\documentclass{article}

\usepackage{microtype}
\usepackage{booktabs} 
\usepackage{hyperref}       
\usepackage{url}            
\usepackage{amsfonts}       
\usepackage{amsmath}
\usepackage{amsthm}
\usepackage{amssymb}
\usepackage[toc, page]{appendix}
\usepackage{array}
\usepackage{authblk}
\usepackage{booktabs}
\usepackage{comment}
\usepackage{bbm}
\usepackage{fullpage}
\usepackage[T1]{fontenc}
\usepackage{mathtools}
\usepackage[utf8]{inputenc}
\usepackage{nicefrac}  
\usepackage[numbers]{natbib}
\usepackage{microtype}      
\usepackage{graphicx}
\usepackage{tabularx, booktabs, makecell, caption}
\usepackage[colorinlistoftodos]{todonotes}
\usepackage{siunitx}
\usepackage{multirow}
\usepackage{enumitem}
\usepackage{color}
\usepackage{wrapfig}
\usepackage{xcolor}
\newtheorem{thm}{Theorem}

\newtheorem{lemma}{Lemma}

\newtheorem{prop}{Proposition}
\newtheorem{claim}{Claim}
\newtheorem{definition}{Definition}
\newtheorem{remark}{Remark}

\usepackage{hyperref}

\begin{document}
\title{Stability of SGD: Tightness Analysis and  Improved Bounds}
\author[1]{Yikai Zhang \thanks{equal contribution}}
\author[1]{Wenjia Zhang$^*$}
\author[2]{Sammy Bald$^*$}
\author[3]{Vamsi Pingali}
\author[4]{Chao Chen}
\author[2]{Mayank Goswami}
\affil[1]{Department of Computer Science, Rutgers University}
\affil[2]{Department of Computer Science, Queens College of CUNY}
\affil[3]{Department of  Mathematics, Indian Institute of Science}
\affil[4]{Department of Biomedical Informatics, Stony Brook University}
\renewcommand\Authands{ and }
\maketitle

\begin{abstract}
Stochastic Gradient Descent (SGD) based methods have been widely used for training large-scale machine learning models that also generalize well in practice. Several explanations have been offered for this generalization performance, a prominent one being algorithmic stability \cite{hardt2016train}. However, there are no known examples of smooth loss functions for which the analysis can be shown to be tight. Furthermore, apart from properties of the loss function, data distribution has also been shown to be an important factor in generalization performance.
This raises the question: is the stability analysis of \cite{hardt2016train} tight for smooth functions, and if not, for what kind of loss functions and data distributions can the stability analysis be improved? 

In this paper we first settle open questions regarding tightness of bounds in the data-independent setting: we show that for general datasets, the existing analysis for convex and strongly-convex loss functions is tight, but it can be improved for non-convex loss functions. Next, we give novel and improved data-dependent bounds: we show stability upper bounds for a large class of convex regularized loss functions, with \emph{negligible  regularization} parameters, and improve existing data-dependent bounds in the non-convex setting. We hope that our results will initiate further efforts to better understand the data-dependent setting under non-convex loss functions, leading to an improved understanding of generalization abilities of deep networks.

\end{abstract}

\section{Introduction}

\emph{Stochastic gradient descent} (SGD) has gained great popularity in solving machine learning optimization problems~\cite{kingma2014adam,johnson2013accelerating}. SGD leverages the finite-sum structure of the objective function, avoids the expensive computation of exact gradients, and thus provides a feasible and efficient optimization solution in large-scale settings~\cite{bottou2012stochastic}.
The convergence and the optimality of SGD have been thoroughly studied ~\cite{ge2015escaping,rakhlin2012making,reddi2018convergence,zhou2019lower,carmon2019lowera,carmon2019lowerb,shamir2013stochastic}. 

In recent years, new research questions have been raised regarding SGD's impact on a model's generalization power. The seminal work \cite{hardt2016train} tackled the problem using the \emph{algorithmic stability} of SGD, i.e., the progressive sensitivity of the trained model w.r.t.~the replacement of a single (test) datum in the training set. They showed that the generalization error of an SGD-trained model is upper bounded by a uniform stability parameter $\varepsilon_{\text{stab}}$, and relate $\varepsilon_{\text{stab}}$ to the divergence of the two parameter vectors obtained by training on twin datasets. This stability-based analysis of the generalization gap allows one to bypass classical model capacity theorems~\cite{vapnic1998statistical, koltchinskii2000rademacher} or weight-based complexity theorems~\cite{neyshabur2017exploring, bartlett2017spectrally, arora2018stronger}.  
This framework also provides theoretical insights into many phenomena observed in practice, e.g.,  
the ``train faster, generalize better'' phenomenon, the power of regularization techniques such as weight decay~\cite{krogh1992simple}, dropout~\cite{srivastava2014dropout}, and gradient clipping.
Other works have developed the stability notion with advanced analysis ~\cite{bassily2020stability,feldman2019high,kuzborskij2018data} and adapted it into more sophisticated settings such as Stochastic Gradient Langevin Dynamics and momentum SGD~\cite{mou18a,chaudhari2019entropy,chen2018stability,Li2020On}.     

Despite the promises of this stability-based analysis, it remains open whether the analysis in ~\cite{hardt2016train} can be further improved to reveal the full potential of the stability method, either in general or for specific data-distributions. 

\begin{table*}[ht]
\begin{center}
\caption{Current landscape of stability bounds. [H] indicates results in~\cite{hardt2016train}, [K] indicates results in \cite{kuzborskij2018data} and * indicates results in this paper. $\beta$ is the smoothness parameter. $\zeta$ is a data-dependent constant defined in Lemma \ref{lem_SGD_data_dep}. $\widehat{\varepsilon}_{\text{stab}}$ is on-average stability defined in Def \ref{avg_stab}. $a$, $b$ are small constants free of $T$ and $n$. We only keep $T$ and $n$ term in the bounds.}
\label{tab:results-summary}
\vskip 0.15in
\renewcommand{\arraystretch}{0.9}
\begin{tabular}{||l||p{.15\textwidth}|p{.14\textwidth}|p{.18\textwidth}|p{.21\textwidth}||}
\hline
\textit{SGD Step Size} &  \multicolumn{2}{c|}{Constant $\alpha_{t} = a/\beta$} & $\alpha_{t}= a/(\beta t)$ & $\alpha_t = b/t$\\
\hline
\textit{Loss function} & Strongly Convex  & Convex & Non-Convex & Non-Convex with $\widehat{\varepsilon}_{\text{stab}}$\\
\hline
\hline
\textit{Upper Bound} & $O(\frac{1}{n})$  [H]  & $O(T/n)$   [H] & $O \left(T^{\frac{a}{1+a}}/n\right)$ [H] $ O\left(T^a/{n^{1+a}}\right)$* & $O\left(T^{\frac{\zeta b}{1+\zeta b}}/n\right)$[K] $O(T^{\zeta b}/n^{1 + \zeta b})^*$ \\
\hline
\textit{Lower Bound} & $\Omega(\frac{1}{n})^*$  & $\Omega(T/n)$* & $\Omega(\frac{T^a}{n^{1+a}})$*& Open\\

\hline
\end{tabular}
\end{center}
\end{table*}

\noindent\textbf{Our results:} We provide three kinds of results (see Table~\ref{tab:results-summary}) that complement each other: a) tight existing lower bounds that show settings where stability analysis cannot be improved further for general datasets, b) weaker lower bounds that hint at a possible improvement, along with complementary improved upper bounds, also for general datasets and c) in settings where existing data-independent analysis cannot be improved, we derive improved data-dependent bounds. Below we summarize some of the existing open questions in this line of research, grouped according to properties of the loss function, along with our results addressing these problems.

\subsection{Convex and Strongly Convex Loss}
The following are the main results presented in \cite{hardt2016train} for convex and strongly-convex loss functions (with certain Lipschitz and smoothness conditions), when optimized using SGD. Here $n$ denotes the size of the sample, $T$ the number of steps in SGD, and $\alpha_t$ the size of the SGD step in the $t$-th iteration.
\vspace{0.1cm}

\noindent{1.} For convex loss functions, the stability is upper bounded by $\sum_{i=1}^{T} \alpha_{t}/n$. The smaller the number of iterations $T$ is, the lower this upper bound. Hence ``train faster, generalize better".

\noindent{2.}  In practice, one often uses constant step size: $\alpha_t=\alpha$. For convex loss functions the upper bound would then scale linearly in the number of iterations $T$, which seems to be too pessimistic. \cite{hardt2016train} show that by adding a $\frac{\mu}{2} \vert\vert w\vert\vert_{2}^{2}$ regularization term to the convex loss function, where $w$ is the vector of weights and $\mu \in \Theta(1)$ is a small constant, one gets a much better stability upper bound for constant step size that does not depend on $T$, and is $O(1/n)$.

This gives rise to the following questions:
\vspace{0.1cm}

\noindent\textbf{Question 1:} Are the upper bounds of \cite{hardt2016train} for convex and strongly-convex functions tight? That is, can one construct loss functions that satisfy the hypotheses and exhibit the claimed worst-case stability performance?

We remark that, to the best of our knowledge, the only construction available in the literature is ~\cite{bassily2020stability}. The authors analyze the stability of a loss function in order to derive lower bounds, but unfortunately the loss function is not smooth and therefore does not satisfy the hypothesis in \cite{hardt2016train}.
\vspace{0.1cm}

\noindent\textbf{Question 2:} How important is the regularization term in order to make the transition from convex to strongly-convex; and therefore the improvement from an $O(T/n)$ upper bound to an $O(1/n)$ upper bound for constant step-size SGD?

We provide the following answers to the above questions:

\noindent\textbf{Result 1:} The answer to question 1 is yes, i.e., there exist smooth, convex and strongly-convex loss functions that achieve the worst-case stability upper bound. 

\noindent\textbf{Result 2: (Data-dependent bounds)} We derive an upper bound on the stability for linear model loss function that is independent of $T$ (the number of iterations), even when the weight $\mu$ of the regularization term is very small (of the order of $1/n^4$), as long as the data satisfies a natural condition related to the Rayleigh quotient. Sharing a similar spirit with \cite{kuzborskij2018data}, our result suggests that the property of distribution plays an important role in generalization of SGD, and nice properties of the data can almost replace the need for regularization.

\subsection{Non-Convex Loss} \cite{hardt2016train} also prove an upper bound for non-convex loss functions, and one wonders again whether the bound is tight. After only being able to prove a slightly weaker lower bound, we realized that this was because one can actually improve the analysis in~\cite{hardt2016train}! 

\noindent\textbf{Result 3:} We provide matching lower and upper bounds on the stability of SGD for non-convex functions, that are tighter than the upper bound in \cite{hardt2016train} for a wide and interesting range of values of $T$ (e.g., when $n<T<n^{10}$).

In the non-convex setting, the bounds in both \cite{hardt2016train} and our Result 3 assume a decreasing step-size $\alpha_t \propto 1/t$ in SGD. However, in practice the constant step-size case is very important. Although it is not derived formally, the techniques in \cite{hardt2016train} can be employed to show an \emph{exponential} upper bound for non-convex loss functions minimized using SGD with constant-size step, raising the question of the existence of a better analysis. 

\noindent{\textbf{Result 4:}} We show that without any additional assumptions on either the loss function or the data distribution, improving on this analysis is hopeless by providing a lower bound that is exponential in $T$. 

\noindent{Data-dependent bounds:} This naturally raises the question of deriving data-dependent bounds on stability in the non-convex setting. The work in \cite{kuzborskij2018data} took the first step in this direction by analyzing SGD using concept of ``average stability'' from ~\cite{bousquet2002stability,shalev2010learnability}, and deriving upper bounds on it. Finally, we show:

\noindent\textbf{Result 5:} The improved analysis for uniform stability of SGD on non-convex and smooth loss functions can also be applied to improve on the result in~\cite{kuzborskij2018data} and obtain a tighter bound for the average stability of SGD.

In summary, we essentially close the open questions of tightness in data-independent settings for all three classes of functions, and improve upper bounds in the data-dependent setting. We hope that our results will initiate further efforts to better understand the data-dependent setting under non-convex loss functions and analyze the conditions under which one can expect better upper bounds on stability and generalization of SGD.


\section{Related Works} \label{related}

 The stability framework suggests that a stable machine learning algorithm results in models with good generalization performance \cite{kearns1999algorithmic,bousquet2002stability,elisseeff2005stability,shalev2010learnability,devroye1979distributiona,devroye1979distributionb,rogers1978finite}. It serves as a mechanism for provable learnability when uniform convergence fails \cite{shalev2010learnability,nagarajan2019uniform}. 
The concept of uniform stability was introduced in order to derive high probability bounds on the generalization error~\cite{bousquet2002stability}. Uniform stability describes the worst case change in the loss of a model trained on an algorithm when a single data point in the dataset is replaced.  In~\cite{hardt2016train}, a uniform stability analysis for \emph{iterative algorithms} is proposed to analyze SGD, generalizing the one-shot version in~\cite{bousquet2002stability}.   Algorithmic uniform stability is widely used in analyzing the generalization performance of SGD~\cite{mou18a,feldman2019high,chen2018stability}. The worst case leave-one-out type bounds also closely connect uniform stability with \emph{differential private learning}~\cite{feldman2018privacy,feldman2020private,dwork2006calibrating,wu2017bolt}, where the uniform stability can lead to provable privacy guarantees. While the upper bounds of algorithmic stability of SGD have been extensively studied, the tightness of those bounds remains open. In addition to uniform stability, an \textit{average stability} of the SGD is studied in \cite{kuzborskij2018data} where the authors provide \textit{data-dependent} upper bounds on stability\footnote{While it is an interesting open problem to get data-dependent lower bounds by lower bounding the average stability, we construct lower bounds on the worst-case stability. Thus our lower bounds are general and not data-dependent.}.  Our analysis framework for deriving improved bounds in ~\cite{hardt2016train} can also be applied to improve the data-dependent stability results in ~\cite{kuzborskij2018data}.

In \cite{bassily2020stability}, a lower bound on the stability of SGD for nonsmooth convex losses is proposed. The lower bound is designed to illustrate the tightness of the stability analysis \emph{without} smoothness assumptions. In this work, we report for the first time lower bounds on the uniform stability of SGD for smooth loss functions. 

Our tightness analysis suggests the necessity of additional assumptions for analyzing the generalization of SGD for deep learning.

\section{Preliminaries} \label{preliminary}
In this section we introduce the notion of uniform stability and  establish notation. We first introduce the quantities \emph{empirical risk}, \emph{population risk}, and \emph{generalization gap}.
Given an unknown distribution $\mathcal{D}$ on labeled sample space $Z=X \times \mathbb{R}$, let $S=\{z_1,...,z_n\}$ denote a set of $n$ samples $z_i=(x_i,y_i)$ drawn i.i.d. from $\mathcal{D}$. Let $w \in \mathbb{R}^{d}$ be the parameter(s) of a model that predicts $y$ given $x$, and let $f$ be a loss function where $f(w;z)$ denotes the loss of the model with parameter(s) $w$ on sample $z$. Let $f(w;S)$ denote the \textit{empirical risk} 
$f(w;S)=E_{z \sim S }[f(w;z)]=\frac{1}{n} \sum_{i=1}^{n} f(w;z_i)$
with corresponding \textit{population risk}
$E_{z \sim \mathcal{D} }[f(w;z)]$.
The \textit{generalization error} of the model with parameter(s) $w$ is defined as the difference between the empirical and population risks:     $$\vert E_{z\sim \mathcal{D}} [f(w;z)] -E_{z\sim S} [f(w;z)]\vert. $$ 
Next we introduce \emph{stochastic gradient descent} (SGD). We follow the setting of \cite{hardt2016train}: starting with initialization $w_{0} \in \mathbb{R}^{d}$, an SGD update step takes the form
$$w_{t+1} = w_t-\alpha_t\nabla_w f(w;z_{i_t}),$$
where $i_t$ is drawn from $[n]=\{1,2,\cdots,n\}$ uniformly and independently in each round. 
The analysis of SGD requires the following crucial properties of the loss function $f(\cdot,z)$ at any fixed point $z$, viewed solely as a function of the parameter(s) $w$:

\begin{definition}[$L$-Lipschitz]
A function $f(w)$ is $L$-Lipschitz if $\forall u,v \in \mathbb{R}^{d}$:
$|f(u)-f(v)| \leq L\|u-v\|$.
\end{definition}

\begin{definition}[$\beta$-smooth] A function $f(w)$ is $\beta$-smooth if $\forall u,v \in \mathbb{R}^{d}$:
$|\nabla f(u)- \nabla f(v)| \leq \beta\|u-v\|$.
\end{definition}

\begin{definition} [$\gamma$-strongly-convex]
A function $f(w)$ is $\gamma$-strongly-convex if $\forall u,v \in \mathbb{R}^{d}$:
$$f(u)>f(v)+ \nabla f(v)^\top [u-v]+\frac{\gamma}{2}\|u-v\|^2.$$
\end{definition}

\begin{definition}[$\rho$-Lipschitz Hessian]
A loss function $f$ has a $\rho$-Lispchitz Hessian if $\forall u, v\in \mathbb{R}^d$, $\|\nabla^2f(u) - \nabla^2f(v)\|\leq\rho\|u-v\|.$
\end{definition}

\paragraph{Algorithmic Stability:}
Next we define the key concept of \textit{algorithmic stability}, which was introduced by~\cite{bousquet2002stability} and adopted by~\cite{hardt2016train}. Informally, an algorithm is \textit{stable} if its output only varies slightly when we change a single sample in the input dataset. When this stability is \textit{uniform} over all datasets differing at a single point, this leads to an upper bound on the generalization gap. We now flesh this out more formally.
\begin{definition}
Two sets of samples $S, S'$ are twin datasets if they differ at a single entry, i.e., $S = \{z_1,...z_i,...,z_n\}$ and $S'=\{z_1,...,z_i',...,z_n\}$.
\end{definition}

Now, let $\mathcal{A}$ be a (possibly randomized) algorithm which is parameterized by a sample $S$ of $n$ datapoints as $\mathcal{A}(S)$. 
\begin{definition}(Stability)
Define the algorithmic stability parameter $\varepsilon_{\text{stab}}(\mathcal{A},n)$ as
\begin{equation*}
   \inf \{\varepsilon:\sup_{{z, S, S'}} \mathbb{E}_{\mathcal{A}}\vert f(\mathcal{A}(S);z)-f(\mathcal{A}(S');z)\vert \leq \varepsilon \}.
\end{equation*}
\end{definition}

The expectation $\mathbb{E}_{\mathcal{A}}$ factors in the possible randomness of $\mathcal{A}$. For such an algorithm, one can define its expected generalization error as $$GE(\mathcal{A},n)\coloneqq \mathbb{E}_{S,\mathcal{A}}[E_{z\sim \mathcal{D}} [f(\mathcal{A}(S);z)] -E_{z\sim S} [f(\mathcal{A}(S^{'});z)]].$$
We also define a data-dependent stability which is an average stability that was introduced by~\cite{rakhlin2005stability,shalev2010learnability} and was applied for analyzing algorithmic stability of SGD by \cite{kuzborskij2018data}.

\begin{definition}[On-average stability]\label{avg_stab}
Let $\mathcal{D}$ be the data distribution and $w_0$ be the initialized weight. A randomized algorithm $\mathcal{A}$ is ${\widehat{\varepsilon}}_{\text{stab}}(\mathcal{D}, w_0)$-on-average stable if
\begin{equation*}
    \mathbb{E}_{S,S'}\mathbb{E}_{\mathcal{A}} [f(\mathcal{A}_S;z) - f(\mathcal{A}_{S^{'}};z)] \leq {\widehat{\varepsilon}}_{\text{stab}}(\mathcal{D}, w_0),
\end{equation*}
where $S\stackrel{iid}{\sim}\mathcal{D}^m$ and $S^{'}$ is its copy with $i$-th example replaced by $z\stackrel{iid}{\sim}\mathcal{D}$.
\end{definition}

Throughout this paper, we will write ${\varepsilon}_{\text{stab}}$ and ${\widehat{\varepsilon}}_{\text{stab}}$ omitting dependencies that are clear in context.
\vspace{0.1cm}

\noindent\textbf{Stability and generalization:} 
It was proved in~\cite{hardt2016train} that $GE(\mathcal{A},n) \leq \varepsilon_{\text{stab}}(\mathcal{A},n)$. Furthermore, the authors observed that an $L$-Lipschitz condition on the loss function $f$ enforces a uniform upper bound: $\sup_{z\in Z}|f(w;z)-f(w';z)| \leq L\|w-w'\|$. This implies that for a Lipschitz loss, the algorithmic stability $\varepsilon_{\text{stab}}(\mathcal{A},n)$ (and hence the generalization error $GE(\mathcal{A},n)$) can be bounded by obtaining bounds on $\|w-w'\|$. And in \cite{kuzborskij2018data} they have similar results in the notion of  on-average stability.

Let $w_{t}$ and $w_{t}^{'}$ be the parameters obtained by running SGD on twin datasets $S,S'$ respectively for $t$ iterations. 
The \emph{divergence quantity} is defined as $\delta_{t}\coloneqq \mathbb{E}_{\mathcal{A}}\vert\vert w_{t}-w_{t}^{'}\vert\vert$. While \cite{hardt2016train} reports upper bounds on $\delta_{t}$ for different loss functions, e.g., convex and non-convex loss functions, we investigate the tightness of those bounds.

\section{Main Results} \label{main}
 In this section we report our main results. We first consider the convex case with constant step size, where we prove 1) that the existing bounds in ~\cite{hardt2016train} are tight, and 2) for linear models, the we report a data-dependent analysis to show that $\varepsilon_{\text{stab}}$ does not increase with $t$. Then we move on to the non-convex case, where a) for decreasing step size we report a lower bound suggests that within a wide range of $T$, existing bound in ~\cite{hardt2016train} is not tight. We prove a tighter upper bound which matches our lower bound thus, and b) for constant step size we give loss functions whose divergence $\delta_{t}$ increases exponentially with $t$.

\subsection{Convex Case}
In this section we analyze the stability of SGD when the loss function is convex and smooth. We begin with a construction which shows that Theorem 3.8 in~\cite{hardt2016train} is tight. 
Our lower bound analysis will require the quadratic function

\begin{equation}\label{quad_fn}
    f(w;z)=\frac{1}{2} w^\top Aw - yx^\top w,
\end{equation}
where $A$ is a $d\times d$ matrix. In the construction of lower bounds, we  carefully choose $A$ and $S$ so that the single data point replaced in the twin data set will cause the instability of SGD.
In particular, we will choose $A$ to be a PSD matrix in the convex case in the construction of the lower bound and choose $A$ to be an indefinite matrix with some strictly negative eigenvalues in the non-convex case. We first begin with the following lemma which describes how $\|w_t-w'_t\|$ behaves for functions defined in Equation~\ref{quad_fn}.

\begin{lemma}[Dynamics of divergence] \label{lem1}
Let $f(w;x) = \frac{1}{2}w^\top Aw - yx^\top w$. Suppose $[x_i-x_i']/\|x_i-x_i'\|$ is an eigenvector of $A$, i.e., $A[x_i-x_i'] = \lambda_{xx'}[x_i-x_i']$. Let $\Delta_{t}$ be $w_t-w_t'$, $\alpha_t \leq \lambda_{xx'}$ be the step size of SGD and $\Delta_0 =0$. If one  runs SGD on $f(w,S)$ and $f(w,S')$ where $S,S'$ are twin datasets and  ${x'}_i^\top x_j = 0, x_i^\top x_j = 0, \;\forall j \neq i$, then the dynamics of $\Delta_t$ are given by

\begin{equation}
\mathbb{E}_{\mathcal{A}} \|\Delta_{t+1}\| = (1-\alpha_t \lambda_{xx'})\mathbb{E}_{\mathcal{A}}\|\Delta_t\|+\frac{\alpha_t}{n}\|x_i-x_i'\|.
\end{equation}
\end{lemma}
The next lemma recursively applies Lemma~\ref{lem1}. We will carefully chose $\lambda_{xx'}$ in the following lemma for lower bound constructions in the convex and non-convex cases. 

\begin{lemma} [Lower bound on divergence] \label{lem2} 
Let $f(w;x) = \frac{1}{2}w^\top Aw - yx^\top w$. Suppose $[x_i-x_i']/\|x_i-x_i'\|$ is an eigenvector of $A$ where $A[x_i-x_i'] = \lambda_{xx'}[x_i-x_i']$. Let $\Delta_{t}$ be $w_t-w_t'$, $\alpha_t \leq \lambda_{xx'}$ be the step size of SGD and $\Delta_0 =0$. If one  runs SGD on $f(w,S)$ and $f(w,S')$ where $S,S'$ are twin datasets and  ${x'}_i^\top x_j = 0, x_i^\top x_j = 0, \;\forall j \neq i$, then we have 
$$\mathbb{E}_{\mathcal{A}}\|\Delta_{T}\| \geq \frac{\|x_i-x_i'\|}{n} \sum_{t=1}^{T-1} \prod_{\tau=t+1}^{T-1} \alpha_t(1-\alpha_\tau\lambda_{xx'}).$$

\end{lemma}

Now  we can present our tightness results. We begin with the convex case. The main idea of the construction is to leverage Equation~\ref{quad_fn} with specially designed $A$ and $S,S'$ to ensure that $\mathbb{E}_{\mathcal{A}}\|w_T-w'_T\|$ will diverge. To obtain the $L$-Lipschitz condition, we trim $f(w;S)$ to mimic the Huber loss function~\cite{huber1992robust} so that the smoothness is maintained for the piecewise function.

\begin{thm} [Lower bound for convex losses]\label{CvxLowerBound}
Let $w_t,w_t'$ be the outputs of SGD on twin datasets $S,S'$ respectively. Let $\Delta_{t} = w_t-w_t'$ and $\alpha_t$ be the step size of SGD. There exists a function $f$ which is convex, $\beta$-smooth, and $L$-Lipschitz, and twin datasets $S,S'$ such that
\begin{equation}
 \varepsilon_{\text{stab}} \geq  \frac{L}{2n} \sum_{t=1}^{T} \alpha_t.
\end{equation}
\end{thm}

The convex upper bound in Theorem 3.8 of~\cite{hardt2016train} states that $\mathbb{E}_{\mathcal{A}}\|\Delta_T\| \leq \sum_{i=1}^{T} \frac{ \alpha_t L}{n}$, which implies that the divergence increases throughout training. The lower bound in Theorem~\ref{CvxLowerBound} suggests the tightness of the upper bound. However, in practice, this is not commonly observed; the generalization performance does not deteriorate as the number of training iterations increases. Under the $\gamma$-strongly-convex loss function condition,~\cite{hardt2016train} provides an $O(\frac{1}{n})$ uniform stability bound, which fits better with empirical observations on classical convex losses. In the next theorem, we show the tightness of the $O(\frac{1}{n})$ bound for strongly-convex losses.

\begin{thm} [Lower bound for strongly-convex losses] \label{strongly_cvx_lowerbound}
Let $w_t,w_t'$ be the outputs of SGD on twin datasets $S,S'$ respectively, $\Delta_{t}$ be $w_t-w_t'$ and $\alpha =\frac{1}{2 \beta}$ be the step size of SGD. There exists a function $f$ which is $\gamma$-strongly-convex and $\beta$-smooth, and twin datasets $S,S'$ such that the divergence and stability of the two SGD outputs satisfies
\begin{equation}
 \varepsilon_{\text{stab}} \geq  \frac{1}{16\gamma n}.
\end{equation}
\end{thm}

Theorem \ref{strongly_cvx_lowerbound} provides evidence for the tightness of the $O(\frac{1}{n})$ stability bound on SGD. To obtain such stability, the loss function must satisfy $\nabla^2_w f(w;z) >\gamma I_d$ with $\gamma=\Omega(1)$. In general this does not hold, e.g., the Hessian of an individual linear regression loss term is $x_j x_j^\top$ which is not strongly-convex. 
In practice one can incorporate a strongly-convex regularizer to impose strong convexity, often resulting in improved generalization performance in practice~\cite{shalev2010learnability,bousquet2002stability}. However, an $O(1)$ regularization term  will bias the loss function away from achieving sufficiently low empirical risk. This motivates us to investigate a weaker condition than strong convexity which still can enforce an $O\left( \frac{1}{n} \right)$ stability, without substantially biasing the loss function.


In the remainder of this section, we restrict ourselves to a family of linear model loss functions and show that the $O(\frac{1}{n})$ stability results can be obtained under the framework of average stability. The results of Theorem~\ref{cvxupperbound} have a dependence on a property of the distribution, and are thus data-dependent.  We begin with the definition of a $\xi$-bounded Rayleigh quotient. Essentially, a bounded Rayleigh quotient dataset requires an average linear dependence of $Span\{x_1,...,x_n\}$. Recall that the $i$-th sample is of the form $z_{i}=(x_{i},y_{i})$.

\begin{definition} \label{self_corr}
A set $S=\{(x_1,y_1),...,(x_n,y_n)\}$ is defined to have $\xi_S$- bounded Rayleigh quotient if $\forall v \in Span\{x_1,...,x_n\}$
$$ v^\top (\frac{1}{n}\sum^{n}_{i=1}x_ix_i^\top)v\geq \xi_S v^\top v.$$

A distribution $\mathcal{D}$ has a $(\xi,n,\mu)$-inversely bounded Rayleigh quotient if there exists a constant $\xi>0$ such that
$$\mathbb{E}_{S\sim \mathcal{D}^n} \left[ \frac{1}{\xi_S+\mu} \right] \leq \frac{1}{\xi+\mu}.$$
\end{definition}

\begin{remark}
The value of $\xi_S$ is always lower bounded by the minimum nonzero eigenvalue of $\frac{1}{n} \sum_{j} x_j x_j^\top $ which is the empirical covariance matrix of sample size $n$. 

\end{remark}


\begin{prop}[Example of distribution with inversely bounded Rayleigh quotient] \label{invers_eigen}
Suppose that $S=\{(x_1,y_1),\ldots,(x_n,y_n)\}$ is sampled from $\mathcal{D}$  with the $x_j$  sampled from a $d$-dimensional spherical Gaussian with dimension $d>10$. Then, $\mathcal{D}$ has a $(\frac{1}{5},n,\mu)$-inversely bounded Rayleigh quotient if $\mu = \Omega(\frac{1}{n^4})$ and $n \geq 2d$.
\end{prop}

\begin{remark}
Proposition~\ref{invers_eigen} implies that for data generated in the form of $\Tilde{x}=UDx$ where $U \in \mathbb{R}^{d \times k}$ is a column-wise orthonormal matrix, $D=diag(\lambda_1,...,\lambda_k)\in {\mathbb{R}_+}^{k \times k}$ is a diagonal matrix and $x \sim \mathcal{N}(0,I_k)$, the empirical covariance matrix has a bounded regularized inverse. Thus distribution of $\Tilde x$  has a  $(\frac{1}{5\lambda_{k}},n,\frac{1}{n^4})$-inversely bounded   Rayleigh quotient. 
\end{remark}
In our next theorem, we leverage the inversely bounded Rayleigh quotient condition to prove a non-accumulated on-average stability bound for SGD on \emph{linear models} with a regularized loss function. We characterize a linear model by rewriting the loss function $f(w;z)$ in terms of $f_y(w^\top x)$ where $f_y(\cdot)$ is a scalar function depending only on the inner product of the model parameter $w$ and the input feature $x$. 

\begin{thm} [Data-dependent stability of SGD with inversely bounded Rayleigh quotient] \label{cvxupperbound}
 Suppose a loss function $f(w;z)$ is of the form  $$f(w;S) =\frac{1}{n} \sum_{j=1}^{n} f_{y_j}(w^\top x_j)+\frac{\mu}{2} w^\top w,$$ where $f_y(w^\top x)$ satisfies $(1)\; |f_y'(\cdot)| \leq L $ , $(2)\; 0<\gamma \leq f_y''(\cdot) \leq \beta$, (3) $S , S'$ are sampled from $\mathcal{D}$ which has  $(\xi,n,\frac{\mu}{\gamma})$ -inversely bounded Rayleigh quotient with bounded support on $x$: $\|x\| \leq R$ and 4) $\mu =\Omega(\frac{\gamma}{n^4})$. Let $w_t$ and $w_t'$ be the outputs of SGD on $S$ and $S'$ after $t$ steps, respectively. Let the divergence $\Delta_{t}=w_t-w_t'$ and $\alpha\leq \frac{\mu
 }{2\beta^2R^2}$ be the step size of SGD.
Then,
\begin{equation}
     {\widehat{\varepsilon}}_{\text{stab}} \leq \frac{16L^2R^2}{\xi\gamma n}.
\end{equation}
\end{thm}

\begin{remark}

The  inversely bounded Rayleigh quotient condition allows SGD to maintain an average stability guarantee for a family of widely used models with a negligible regularizer and large sample size. The theorem suggests that if the dataset $S$ is sampled from a `good' distribution, one can obtain an advanced generalization property which mainly depends on the distribution. The theorem also justifies the common choice of small values for the weight in the $L_2$-regularizer (also known as \textit{weight decay}) when training ridge regression type models.

\end{remark}

\textbf{Example: Linear regression.} Linear regression minimizes the quadratic loss on $w$: $f(w,S) = \frac{1}{2n}\sum_{x_j\in S} (x_j^\top w-y_j)^2$. Note that the Hessian of an individual linear regression loss term is $x_j x_j^\top$ which is \textit{not strongly-convex}. However, one can rewrite the loss function as $f_y(w^\top x)$ where $f''_y(\cdot) = 1 $. Hence Theorem~\ref{cvxupperbound} can be applied to give a data-dependent bound on the stability of SGD in above example.

\subsection{Non-Convex Case}
In this section, we construct a non-convex loss function to analyze the tightness of the divergence bound in \cite{hardt2016train}.  We first focus on the case where SGD applies a step size that \textit{decreases with $t$}. Define a \emph{hitting time} to be the time $t$ that satisfies $w_{t-1}-w_{t-1}^{'} = 0$ and $w_{t}-w_{t}^{'} \neq 0$. We first fix a hitting time $t_{0}$ and prove Lemma~\ref{thm2}.

\begin{lemma}[Divergence of non-convex loss function]\label{thm2}
There exists a function $f$  which is non-convex and $\beta$-smooth,  twin datasets $S,S'$ and constant $a>0$ such that the following holds: if SGD is run using step size $\alpha_t = \frac{a}{0.99 \beta t}$ for $1 \leq t < T$, and $w_t, w_t'$ are the outputs of SGD on $S$ and $S'$, respectively, and $\Delta_{t}=w_t-w_t'$, then
$$
    \forall 1 \leq t_0 \leq T,\ \ \ \ 
    \mathbb{E}_{\mathcal{A}}\left[ \|\Delta_T\| | \Delta_{t_0} \neq 0  \right] \geq \frac{1}{2n} \left(\frac{T}{t_0}\right)^a.
$$
\end{lemma}
The following theorem follows from Lemma~\ref{thm2} by optimizing over $t_{0}$. The choice of hitting time $t_0$ plays an important role in the analysis, which is also illustrated in the ``burn-in Lemma'' 3.11 in \cite{hardt2016train}.

\begin{thm}[Lower bound for non-convex loss functions]\label{thm3}
Let $w_t, w_t'$ be the outputs of SGD on twin datasets $S,S'$, and $\Delta_{t}=w_t-w_t'$. There exists a function $f$ which is non-convex and $\beta$-smooth, twin datasets $S,S'$ and  constants $a<0.1$ such that the divergence of SGD after $T>n$ rounds using constant step size $\alpha_t = \frac{a}{0.99\beta t}$ satisfies
\begin{equation}\varepsilon_{\text{stab}} \geq \frac{T^{a}}{6n^{1+a}}.\end{equation}
\end{thm}

In the above theorem, the lower bound is derived by  choosing $t_0=n$ in Lemma~\ref{thm2}. The bound in \cite{hardt2016train} is of the form $O\left( \frac{T^{\frac{a}{1+a}}}{n}\right)$ which does not match the above lower bound.  According to the lower bound provided in Theorem \ref{thm3},
the bound in \cite{hardt2016train} may not be tight in the region $T^{\frac{a}{1+a}} \leq n$. We investigate this gap and derive a tighter bound in the next theorem which improves on Theorem 3.12 in~\cite{hardt2016train}.

To prove a better upper bound for non-convex losses, we first consider the case of sampling from the data without replacement, which we call \emph{permutation SGD}. We need the following lemma, which gives us the expectation of divergence for a given hitting time $t_k + 1$, which is the timestamp of permutation SGD first selecting the $k$-th different sample.

\begin{lemma} \cite{hardt2016train}\label{lem_SGD}
	Assume $f$ is $\beta$-smooth and $L$-Lipschitz. Let $w_t, w_t'$ be outputs of $SGD$ on twin datasets $S,S'$ respectively after $t$ iterations and let $\Delta_{t}=[w_t-w_t']$ and $\delta_t = \mathbb{E}\|\Delta_t\|$. Running SGD on $f(w;S)$ with step size $\alpha_t = \frac{a}{\beta t }$ satisfies the following conditions:
	\begin{itemize}
		\item 
		The SGD update rule is a $(1+\alpha_t \beta)$-expander and $2\alpha_t L$-bounded. 
		\item 
		$\mathbb{E}_{\mathcal{A}}[ \|\Delta_{t}\| | \Delta_{t-1} ]  \leq \left(1+\alpha_t\beta \right)\|\Delta_{t-1}\| +\frac{2\alpha_tL}{n}$.
		\item 
		$\mathbb{E}_{\mathcal{A}}[ \|\Delta_{T}\| | \Delta_{t_{k}}=0 ] \leq \big(\frac{T}{t_{k}}\big)^{ a}\frac{2L}{ n}$.
	\end{itemize}
\end{lemma}
 By taking the expectation over hitting time $t_k + 1$ from 0 to $n$ we obtain an upper bound on the stability for non-convex losses.
 
\begin{thm} [Permutation SGD]\label{permutation}
	Assume $f$ is $\beta$-smooth and $L$-Lipschitz. Running $T>n$ iterations of  SGD  on $f(w;S)$ with step size $\alpha_t = \frac{a}{\beta t }$, the stability of SGD satisfies\\ 
	\begin{equation}
	 \varepsilon_{\text{stab}} \leq  \frac{2L^2T^a}{n^{1+a}}.
	\end{equation}
\end{thm}
Dividing our bound by the bound in Theorem 3.12 of~\cite{hardt2016train}, we obtain the ratio $ \Omega \left( \frac{T^{\frac{a^2}{1+a}}}{n^a}\right)$. This factor is less than 1 (and so we improve the upper bound) exactly when $T^{\frac{a}{1+a}} \leq n$. Note that this is potentially a large range as $a$ is a small and positive constant. We remark that our tight bound is for \emph{permutation SGD}. We also prove the bound for \emph{uniform sampling SGD} which uses sampling with replacement with an additional $\log(n)$ factor, and still achieves a polynomial improvement for a wide range of $T$. 

\begin{lemma} \label{lem_prob_rule}
	Let $w_t, w_t'$ be outputs of $SGD$ on twin datasets $S,S'$  after $t$ iterations and let $\Delta_{t}=w_t-w_t'$. Suppose that $t_k = ct_{k-1}$. Then the following conditions hold:
	
	\begin{itemize}
		\item 
		$\mathbb{P}[ \Delta_{t_k-1}=0| \Delta_{t_{k}}\neq 0 ]\leq \frac{n}{n+t_{k-1}}$.
		\item 
		$\mathbb{P}[ \Delta_{t_k-1} \neq0| \Delta_{t_{k}}\neq 0 ]\leq \frac{1}{c}\left(1+\frac{t_k}{n}\right)$.
		\item
	    	$\mathbb{E}_{\mathcal{A}}[ \|\Delta_{T}\|| \Delta_{t_{k}}\neq 0 ]
		\leq \frac{1}{c}\left(1+\frac{t_k}{n}\right)\mathbb{E}_{\mathcal{A}}[ \|\Delta_{T}\| | \Delta_{t_{k-1}}\neq 0 ] 
		    	+\big(\frac{T}{t_{k-1}}\big)^{ a}\frac{2L}{ n}$.
		
	\end{itemize}
\end{lemma}
By applying the last inequality in Lemma~\ref{lem_prob_rule} recursively, we could bound the case where the hitting time is not equal to $t_k.$ Then we obtain an upper bound for the stability of uniform sampling SGD as follows:

\begin{thm}[Uniform sampling SGD]\label{unif_indep}
    Assume $f$ is $\beta$-smooth and $L$-Lipschitz. Running $T>n$ iterations of SGD  on $f(w;S)$ with step size $\alpha_t = \frac{a}{\beta t }$, the stability of SGD satisfies\\ 
	$$\varepsilon_{\text{stab}} \leq  16\log(n)L^2\frac{T^a}{n^{1+a}}.
	$$
\end{thm}

In~\cite{kuzborskij2018data}, the data-dependent stability of SGD is analyzed, incorporating the dependence on the variance of SGD  curvature and the loss of the initial parameter $w_0$ in analyzing the divergence of SGD. This framework has applications in transfer learning, as well as implications including optimistic generalization error. We observe that our analysis in Theorems ~\ref{permutation} and~\ref{unif_indep} can be combined with the data-dependent framework, and we now report our data-dependent versions of Theorems~\ref{permutation} and~\ref{unif_indep}.
The analysis requires the additional bounded variance assumption for SGD which we now present:
In the rest of this section we assume the variance of SGD satisfies
\[
\mathbb{E}_{S, z}\left[\|\nabla f(w_{t}; z) - \nabla \mathbb{E}_{z}(f(w_{t}; z))\|^2\right] \leq \sigma^2,\;\;\;\; \forall t.
\]

We borrow the following lemma from~\cite{kuzborskij2018data} which is a data-dependent version of Lemma~\ref{lem_SGD}. 

\begin{lemma} \cite{kuzborskij2018data}\label{lem_SGD_data_dep}
    Assume $f$ is $\beta$-smooth, $L$-Lipschitz, and has a $\rho$-Lipschitz Hessian. With $w_0$ the initial weight and $w_t$, $w_{t'}$ the outputs of SGD on twin datasets $S,S'$ respectively after $t$ iterations, let $\Delta_t =[w_t - w_{t'}]$. Running SGD on $f(w;S)$ with step size $\alpha_t = \frac{b}{t}$ where $b \leq \min\{\frac{2}{\beta}, \frac{1}{8\beta^2\ln T^2}\}$ has the following properties:
    \begin{itemize}
        
        \item The SGD update rule is a $(1+\alpha_t\psi_t)$-expander and $\alpha_t L$-bounded. Here $\psi_t = \min\{\beta, \kappa_t\}$ where 
        
        $$\kappa_t = \|\nabla^2 f(w_0;z_t)\|_2 + \frac{\rho}{2}\|\sum_{k=1}^{t-1}\alpha_k\nabla f(w_{S, k}; z_k)\|+ \frac{\rho}{2}\|\sum_{k=1}^{t-1}\alpha_k\nabla f(w_{S^{'}; k}, z_k)\|.
        $$
        
        \item $\mathbb{E}_{\mathcal{A}}[\|\Delta_{t+1}\||\Delta_{t_0}=0] \leq\{\mathbb{E}_{\mathcal{A}}[\|\Delta_{t}\||\Delta_{t_0}=0][1+(1-\frac{1}{n})\alpha_t\psi_t]\} + \frac{2\alpha_tL}{n}.$
        
        \item
        $E_{S,S^{'}}\{E_{\mathcal{A}}[\|\Delta_{T}\||\Delta_{t_0}=0]\} \leq \frac{L}{n}\left(\frac{T}{t_0}\right)^{\zeta b}$, where
        
            $$\zeta  = \tilde{O}(\min\{\beta, E_z[\|\nabla^2f(w_0; z)\|_2] + \Delta^*_{1, \sigma^2}\})$$
           $$\Delta^*_{1, \sigma^2}   = \rho(b\sigma + \sqrt{bE_z[f(w_0; z)] - \inf_{w}E_z[f(w; z)].}$$
            
    \end{itemize}
\end{lemma}


Similar to the proof of Theorem~\ref{permutation}, we apply Lemma~\ref{lem_SGD_data_dep} to show a data-dependent version of Theorem~\ref{permutation}.

\begin{thm}[Data-dependent version of Theorem~\ref{permutation}]\label{permutation_new}
    Assume $f$ is $\beta$-smooth, $L$-Lipschitz, and has a $\rho$-Lipschitz Hessian. Let $w_0$ be the initial weight and $w_t$, $w_{t'}$ be the outputs of SGD on twin datasets $S$ and $S'$ respectively after $t$ iterations. Let $\Delta_t =[w_t - w_{t'}]$. 
    Running SGD on $f(w;S)$ with step size $\alpha_t = \frac{b}{t}$ where $b \leq \min\{\frac{2}{\beta}, \frac{1}{8\beta^2\ln T^2}\}$ satisfies\\ 
	\begin{equation}
	 {\widehat{\varepsilon}}_{\text{stab}} \leq  \frac{L^2T^{\zeta b}}{\zeta n^{1+\zeta b}}.
	\end{equation}
\end{thm}
We could obtain the ratio $\Omega(T^{\frac{(\zeta b)^2}{1+\zeta b}}/(\mathbb{E}_{S,\mathcal{A}}[f(w_T;S)]^\frac{1}{1+ \zeta b}n)^{\zeta b})$ by dividing our stability bound in the results of Theorem 4 of~\cite{kuzborskij2018data}. This factor is less than 1 when $T^{\frac{\zeta b}{1+\zeta b}}<\mathbb{E}_{S,\mathcal{A}}[f(w_T;S)]^\frac{1}{1+ \zeta b}n$. Since $b\leq \min\{2/\beta, 1/(8\beta^2\ln T^2)\}$ and $\zeta$ is bounded above by $\beta$, and $\mathbb{E}_{S,\mathcal{A}}[f(w_T;S)]$ is usually $\Theta(1)$, within a large range of $T$ we have a polynomial improvement over Theorem 4 of~\cite{kuzborskij2018data}.

The following lemma is a direct application of Lemma \ref{permutation}. It is also an on-average extension of Lemma~\ref{lem_prob_rule} part 3.
\begin{lemma}[Data-dependent version of Lemma~\ref{lem_prob_rule}]\label{lem_prob_rule_new}
	Let $w_t, w_t'$ be outputs of $SGD$ on twin datasets $S,S'$ respectively after $t$ iterations and let $\Delta_{t}=w_t-w_t'$. And let $b$, $\zeta$ be as in Lemma~\ref{lem_SGD_data_dep}. Suppose that $t_k = ct_{k-1}$. Then the following condition holds:
	\begin{equation}
	     \mathbb{E}_{S,S^{'}}\mathbb{E}_{\mathcal{A}}[ \|\Delta_{T}\|| \Delta_{t_{k}}\neq 0 ]
		\leq \big(\frac{T}{t_{k-1}}\big)^{\zeta b}\frac{L}{\zeta n} 
		 + \mathbb{E}_{S, S{'}}\mathbb{E}_{\mathcal{A}}[ \|\Delta_{T}\| | \Delta_{t_{k-1}}\neq 0 ] \frac{1}{c}\left(1+\frac{t_k}{n}\right). 
	\end{equation}
\end{lemma}

Based on the above lemma, we can prove an upper bound of on-average stability with uniform sampling SGD using the same technique as for Theorem~\ref{unif_dep}.

\begin{thm}[Data-dependent version of Theorem~\ref{unif_indep}] \label{unif_dep}
    Assume $f$ is $\beta$-smooth, $L$-Lipschitz, and has a $\rho$-Lipschitz Hessian. Let $w_t,w_{t'}$ be the outputs of SGD on twin datasets $S,S'$ respectively after $t$ iterations and let $\Delta_t =[w_t - w_{t'}]$ and $\delta_t = E_\mathcal{A}\|\Delta_t\|$. And let $\zeta$ follow the same definition as in Lemma~\ref{lem_SGD_data_dep}. Running SGD on $f(w;S)$ with step size $\alpha_t = \frac{b}{t}$ where $b < 1$ satisfies\\ 
    \begin{equation}
	{\widehat{\varepsilon}}_{\text{stab}} \leq  \frac{16\log(n)L^2T^{\zeta b}}{\zeta n^{1+\zeta b}}.
	\end{equation}
\end{thm} 
We conclude this section with the following lower bound on the uniform stability of SGD with constant stepsize for non-convex loss functions. We show that for non-convex functions satisfying classical conditions  $\beta$-smooth, we cannot avoid a pessimistic bound. Thus, in order to analyze the generalization power of SGD for deep learning loss functions from an optimization perspective, different conditions are necessary.

\begin{thm}  \label{noncvx_lowerbound}
Let $w_t, w_t'$ be the outputs of SGD on twin datasets $S,S'$, and let $\Delta_{t}=w_t-w_t'$. There exists a non-convex, $\beta$-smooth function $f$, twin sets $S,S'$ and constants $a,\gamma$ such that the divergence of SGD after $T>n$ rounds using constant step size $\alpha = \frac{a}{0.99 \gamma }$ satisfies $$\varepsilon_{\text{stab}} \geq \exp(aT/2)/n^2$$
\end{thm}
\section{Conclusion and Future Work}\label{conclusion}

We first provided matching upper and lower data-independent bounds on the stability of SGD for three kinds of loss functions: convex, strongly-convex, and non-convex, essentially closing the gap in all cases. We then provided stronger data-dependent generalization bounds for both convex and non-convex loss functions by analyzing average-stability, showing that nice properties of data can both improve generalization and also reduce the need for regularization. At least two interesting open questions arise from our work: a) Can one obtain data-dependent lower bounds on average-stability that show the tightness of existing analysis? b) Can one devise properties of data-distributions or loss functions (perhaps motivated by deep learning) that imply better data-dependent stability bounds? 

\bibliographystyle{plain}
\bibliography{example_paper}

\begin{appendices}
\setcounter{thm}{0}
\setcounter{lemma}{0}
\setcounter{obs}{0}
\setcounter{prop}{0}

\begin{lemma}[Dynamics of divergence]
Let $f(w;x) = \frac{1}{2}w^\top Aw - yx^\top w$. Assume $y_i = y_i^{'}=1$ for all $i$. Suppose $[x_i-x_i']/\|x_i-x_i'\|$ is an eigenvector of $A$ where $A[x_i-x_i'] = \lambda_{xx'}[x_i-x_i']$. Let $\Delta_{t}=w_t-w_t'$, $\alpha_t \leq \lambda_{xx'}$ be the step size of SGD and $\Delta_0 =0$. Suppose one  runs SGD on $f(w;S)$ and $f(w;S')$ where $S,S'$ are twin datasets and  ${x'}_i^\top x_j = 0, x_i^\top x_j = 0, \;\forall j \neq i$, the dynamics of $\Delta_t$ are given by:

\begin{equation}
\mathbb{E}_{\mathcal{A}} \|\Delta_{t+1}\| = (1-\alpha_t \lambda_{xx'})\mathbb{E}_{\mathcal{A}}\|\Delta_t\|+\frac{\alpha_t}{n}\|x_i-x_i'\|
\end{equation}
\end{lemma}

\begin{proof}

In case the different entry $z_i,z_i$ is not picked,
the gradient difference of $f(w;z)$ and $f(w;z')$ is 
$$\nabla f(w;z) - \nabla f(w';z') = A[w-w']$$
and in case  different entry $z_i,z_i$ is picked,
$$\nabla f(w;z) - \nabla f(w';z') = A[w-w'] + [x_i-x_i']$$
Since $\Delta_0 = 0$, one can inductively show $\Delta_{t} = \theta_t [x_i-x_i']$ where $\theta_t >0$. 
Since SGD selects $z_t = z_t'$ with probability $1-\frac{1}{n}$ and a different entry with probability $\frac{1}{n}$ we have the following dynamic:

\begin{equation}\label{divergence_eq}
        \Delta_{t+1}= \left\{
        \begin{array}{ll}
          (I-\alpha_t A)[w_t-w'_t]   & \quad  \text{with prob. } 1-\frac{1}{n} \\
        (I-\alpha_t A)[w_t-w'_t] + \alpha_t[x_i-x_i'] & \quad \text{with prob } 1/n.
        \end{array}
        \right.
\end{equation}

\begin{equation*}
    \begin{aligned}
    \mathbb{E}_{\mathcal{A}}\|\Delta_{t+1}\| =& \mathbb{E}_{\mathcal{A}} \left[\|\Delta_{t+1}\| | \text{Index $i$ is not selected} \right]\mathbb{P}[\text{Index $i$ is not selected}]\\
    &+ \mathbb{E}_{\mathcal{A}} \left[\|\Delta_{t+1}\| |  \text{Index $i$ is selected} \right]\mathbb{P}[ \text{Index $i$ is selected}]\\
    =& (1-\frac{1}{n}) \|(I-A)[w_t-w'_t]\| + \frac{1}{n}\|(I-A)[w_t-w'_t] + \alpha_t[x_i-x_i']\|\\
    =&(1-\frac{1}{n})(1-\alpha_t \lambda_{xx'})\theta_t\|x_i-x_i'\|+\frac{1}{n}[1-\alpha_t\lambda_{xx'}\theta_t+\alpha_t]\|x_i-x_i'\|\\
    =&(1-\alpha_t \lambda_{xx'})\mathbb{E}_{\mathcal{A}}\|\Delta_t\|+\frac{\alpha_t}{n}\|x_i-x_i'\|
    \end{aligned}
\end{equation*}
\end{proof}

\begin{lemma} [Lower bound on divergence]
Let $f(w;x) = \frac{1}{2}w^\top Aw - yx^\top w$. Assume $y_i = y_i^{'}$ for all $i$. Suppose $[x_i-x_i']/\|x_i-x_i'\|$ is an eigenvector of $A$ where $A[x_i-x_i'] = \lambda_{xx'}[x_i-x_i']$. Let $\Delta_{t}$ be $w_t-w_t'$, $\alpha_t \leq \lambda_{xx'}$ be the step size of SGD and $\Delta_0 =0$. Suppose one  runs SGD on $f(w;S)$ and $f(w;S')$ where $S,S'$ are twin datasets and  ${x'}_i^\top x_j = 0, x_i^\top x_j = 0, \;\forall j \neq i$, we have 
$$\mathbb{E}_{\mathcal{A}}\|\Delta_{T}\| \geq \frac{\|x_i-x_i'\|}{n} \sum_{t=1}^{T-1} \prod_{\tau=t+1}^{T-1} \alpha_t(1-\alpha_\tau\lambda_{xx'})$$

\end{lemma}

\begin{proof}
By iterative applying Lemma \ref{lem1} we have 
\begin{equation*}
\begin{aligned}
\mathbb{E}_{\mathcal{A}}\|\Delta_{T}\|&=(1-\alpha_{T-1} \lambda_{xx'})\mathbb{E}_{\mathcal{A}}\|\Delta_{T-1}\|+\frac{\alpha_{T-1}}{n}\|x_i-x_i'\|\\
&= \|[x_i-x_i']\|\frac{1}{n} \sum_{t=1}^{T-1} \alpha_t \prod_{\tau=t+1}^{T-1} (1-\alpha_\tau\lambda_{xx'}) 
\end{aligned}
\end{equation*}
\end{proof}

\begin{thm} 
Let $w_t,w_t'$ be the outputs of SGD on twin datasets $S,S'$ respectively, $\Delta_{t}$ be $w_t-w_t'$ and $\alpha_t$ be the step size of SGD initialized with $w_0=w_0'=0$. There exists a function $f$ which is convex and $\beta$-smooth, $L$-Lipschitz on domain of $w_t,w_t'$ and  twin datasets $S,S'$ such that the divergence of the two SGD outputs satisfies:
\begin{equation}
    \mathbb{E}_{\mathcal{A}}\|\Delta_T\| \geq \frac{1}{n} \sum_{t=1}^{T} \alpha_t; \;\;\; \varepsilon_{stab} \geq  \frac{L}{2n} \sum_{t=1}^{T} \alpha_t
\end{equation}
\end{thm}

\begin{proof}
The sketch of the proof is as follows: we construct a Huber function~\cite{huber1992robust} so that 
\begin{enumerate}
    \item The function is quadratic within certain region to ensure the divergence of SGD.
    \item By carefully choosing the function, SGD will never step out the quadratic region.
    \item The function is linear outside the region to ensure the global Lipschitzness.
\end{enumerate}
We start with constructing the quadratic part. Let $f(w;z) = \frac{1}{2}w^\top Aw - yx^\top w$. We choose $A=U\Sigma_KU^\top\in\mathbb{R}^{d\times d}$ to be a symmetric PSD matrix with rank $K$ where $K<d$. Let $U=[u_1,\ldots,u_K]$ be an orthorgonal matrix representing eigenvectors of $A$ and $\Sigma = diag[\lambda_1,\ldots, \lambda_K]$ where $\lambda_1\geq\lambda_2\geq\ldots\geq\lambda_K$ are non-zero eigenvalues of $A$. For twin datasets $S = \{z_1,\ldots,z_n\}$ and $S'=\{z_1,\ldots,z'_i,\ldots,z_n\}$, define $z_i = (v, 0.5)$ and $z'_i = (-v, 0.5)$ where $v^\top Av = 0$. For the rest of the data, $z_j = (x_j, 1)$ for any $j\neq i$ where $x_j$ are unit vectors that lie in the column space of $A$. Let $\lambda_1=2$ and $\lambda_K=1$. And the SGD update follows $w_{t+1} = w_t - a_t(Aw_t - yx)$ with initialization $w_0 = 0$.

\textbf{Claim 1}: $\|U^\top w_t\|\leq \frac{1}{\lambda_K}$ $\forall\ t$

\textbf{Proof}:
We will proof this claim by induction. For $t = 0$, $w_0 = 0$ the conclusion holds.

Suppose for $w_{t}$ the claim holds. We have
\begin{equation*}
    \begin{aligned}
    \|U^\top w_{t+1}\| & = \|U^\top [(I-\alpha_tA)w_t + \alpha_tyx]\|\\
     & = \|(I-\alpha_t\Sigma)U^\top w_t + \alpha_tyU^\top x\|   \\
     & \leq (1-\alpha_t\lambda_K)\|U^\top w_t\| + \alpha_t\|U^\top x\|    \\
     & \leq \frac{1}{\lambda_K}
    \end{aligned}
\end{equation*}

Next we will proof that in this bounded region, the weight divergence is lower bounded by the summation of step size.

\textbf{Claim 2}: Suppose $w_0=w'_0 = 0$, $\mathbb{E}_{\mathcal{A}}\|w_T-w'_T\|=\frac{1}{n}\sum_{t=1}^T\alpha_t$.

\textbf{Proof}:
Following the proof in Lemma~\ref{lem1} and Lemma~\ref{lem2}, we could obtain the result. 

By Claim 1 and Claim 2 we know that with zero initialization, SGD is bounded in the region $\|\Sigma^{\frac{1}{2}}U^\top w\|\leq\frac{1}{\sqrt{\lambda_K}}$ for all $t$. And the weight divergence is lower bounded in this area by $\frac{1}{n}\sum\alpha_t$.

Last, we will define $f(w;z)$ outside the $\|\Sigma^{\frac{1}{2}}U^\top w\|\leq\frac{1}{\sqrt{\lambda_K}}$ region and ensure the global Lipschitzness. By define 
\[
f(w, z) = \mathbbm{1}_{||\Sigma^\frac{1}{2}U^\top w||> \frac{1}{\sqrt{\lambda_k}}} \left\{\frac{1}{\sqrt{\lambda_k}}(||\Sigma^\frac{1}{2}U^\top w||-\frac{1}{2\sqrt{\lambda_k}}) - yx^\top w\right\}
\]
the global Lipschitzness is ensured by choosing $L\leq \frac{1}{\sqrt{\lambda_K}}$.

So the final function $f(w; S)$ is 
\begin{equation*}
    \begin{aligned}
     f(w;S) & =\frac{1}{n}\sum_{j=1}^{n} f(w;x_j,y_j) \\
      & = \mathbbm{1}_{||\Sigma^\frac{1}{2}U^\top w||\leq \frac{1}{\sqrt{\lambda_k}}}\frac{1}{2} w^\top A w 
     + \mathbbm{1}_{||\Sigma^\frac{1}{2}U^\top w||> \frac{1}{\sqrt{\lambda_k}}} \frac{1}{\sqrt{\lambda_k}}(||\Sigma^\frac{1}{2}U^\top w||-\frac{1}{2\sqrt{\lambda_k}}) - \frac{1}{n}\sum_{j=1}^{n} y_jx_j^\top w
    \end{aligned}
\end{equation*}
 
\end{proof}

\begin{thm}[Lower bound for strongly-convex losses]
Let $w_t,w_t'$ be the outputs of SGD on twin datasets $S,S'$ respectively, $\Delta_{t}$ be $w_t-w_t'$ and $\alpha = \frac{1}{2 \beta}$ be the step size of SGD. There exists a function $f$ which is $\gamma$ strongly convex and $\beta$-smooth, $L$-Lipschitz on domain of $w_t,w_t'$ and  twin datasets $S,S'$ such that the divergence and stability of the two SGD outputs satisfies:
\begin{equation}
    \mathbb{E}_{\mathcal{A}}\|\Delta_T\| \geq  \frac{1}{16\gamma n}; \;\;\; \varepsilon_{stab} \geq  \frac{1}{16\gamma n}
\end{equation}
\end{thm}

\begin{proof}
Similar to Theorem~\ref{CvxLowerBound}'s, we will construct $S$, $S'$ and $f(w; z)$ as follows:
\begin{enumerate}
    \item Let $A$ be a positive definite matrix with minimum eigenvalue to be $\gamma$ and maximum eigenvalue bounded by $\beta$. And let the eigenvector corresponding with the minimum eigenvalue to be $v$ and $\|v\| = 1$. Let $\gamma  = \frac{\beta}{2}$. We have the function $f(w;z) = \frac{1}{2}w^\top A w - yx$. 
    \item Define the twin datasets $S$ and $S'$ to be $z_j = (x_j, 0.5)$ where $x_j^\top v = 0$ and $\|x_j\|=1$ for all $j\neq i$. And let $z_i = (v, 0.5)$ and $z'_i = (-v, 0.5)$.
\end{enumerate}

In this setting, we have the similar observation as equation~\ref{divergence_eq} in Lemma~\ref{lem1}. We have
\begin{equation*}
        \Delta_{t+1}= \left\{
        \begin{array}{ll}
          (I-\alpha_t A)[w_t-w'_t]   & \quad  \text{with prob. } 1-\frac{1}{n} \\
        (I-\alpha_t A)[w_t-w'_t] + \frac{\alpha_t}{2}[x_i-x_i'] & \quad \text{with prob } 1/n.
        \end{array}
        \right.
\end{equation*}

Then by induction, we could obtain that with $w_0=w'_0=0$, $\Delta_t = v\theta_t$, where $\theta_t>0$ for $t>0$. Let $\tau$ be the first time that $x_i,x_i'$ are picked, we have $\Delta_{\tau+1}= \frac{\alpha}{2}[x_i-x_i'] = v \alpha_{\tau}$. The iterative step of $\Delta_{t+1}$ and $\Delta_{t}$ implies that $\Delta_{t+1}= v \theta_{t+1}$ where $\theta_{t+1} = (1-\alpha\gamma) \theta_t$ with probability $(1-\frac{1}{n})$ and $\theta_{t+1} = (1-\alpha \gamma) \theta_t+ \alpha_t$ with probability $\frac{1}{n}$ .

The above construction then yields:
\begin{equation}\label{exp_div}
\begin{aligned}
    \mathbb{E}_{1:{t+1}}\left[ \|\Delta_{t+1}\| | \Delta_{t_0} \neq 0  \right] = &\mathbb{E}_{1:{t}} \left[\left( 1-\frac{1}{n}\right)\|(I-\alpha A)\Delta_{t}\| + \frac{1}{n} \| (I- \alpha A)\Delta_{t} + \alpha v\| \right]\\ 
    & =\|v\| \mathbb{E}_{1:{t}}   \left[ \left( 1-\frac{1}{n}\right)(1+\alpha \beta )\theta_t+ \frac{1}{n} ((1+\alpha \beta)\theta_t+ \alpha )  \right] \\
    & = \|v\| \mathbb{E}_{1:{t}} \left[\left[ (1-\alpha \gamma )\theta_t \right] + \frac{ \alpha}{n}   \right]
\end{aligned}
\end{equation}
By literately applying Equation~\ref{exp_div}, we have $\mathbb{E}_{\mathcal{A}} [\|\Delta_T\|] = \theta_T \geq \frac{\left(1-(\frac{3}{4})^{T}\right)}{\gamma n} \geq \frac{1}{16\gamma n}$. 

Next we show that $f(w_T;z)-f(w_T';z) = z^\top [w_T-w_T']$.

In case the $i$-th sample is picked: $w_{t+1}^\top v = (I-\alpha)v^\top Aw_t- \frac{\alpha}{2}$
 and ${w'_{t+1}}^\top v = (I-\alpha)v^\top Aw'_t+\frac{\alpha}{2}$. In case $i$-th sample is not picked $w_{t+1}^\top v =(1-\alpha\gamma)w_t^\top v$ and ${w'_{t+1}}^\top v =  (1-\alpha\gamma){w'_t}^\top v$. Therefore, by induction one can show $w_{t+1}^\top v = -{w'_{t+1}}^\top v$. 
 
 By the  fact that $\Delta_t = \theta_t v$, we know $w_{t+1}^\top v^{\perp} ={w'_{t+1}}^\top v^{\perp}  $. Combing the fact that $w_{t+1}^\top v = -{w'_{t+1}}^\top v$ and $w_{t+1}^\top v^{\perp} ={w'_{t+1}}^\top v^{\perp}$ we have $w_t^\top Aw_t = {w'_t}^\top A w'_t$ which implies  $f(w_T;z)-f(w_T';z) = z^\top [w_T-w_T']$ by the construction of $f(w_T;S)$. Hence we have
 \[
 \sup\limits_{z} \mathbb{E}_{\mathcal{A}} [f(w_T;z)-f(w'_T;z)] = \mathbb{E}_{\mathcal{A}}  [w_T-w'_T]^\top v = \theta_T\geq \frac{1}{16\gamma n}
 \]
\end{proof}


\begin{thm}[Data-dependent stability of SGD with inversely bounded Rayleigh quotient]
 Suppose a loss function $f(w,z)$ is of the form  $$f(w,S) =\frac{1}{n} \sum_{j=1}^{n} f_{y_j}(w^\top x_j)+\frac{\mu}{2} w^\top w,$$ where $f_y(w^\top x)$ satisfies $(1)\; |f_y'(\cdot)| \leq L $ , $(2)\; 0<\gamma \leq f_y''(\cdot) \leq \beta$, (3) $S , S'$ are sampled from $\mathcal{D}$ which is $(\xi,n,\frac{\mu}{\gamma})$ -inversely Rayleigh quotient with a bounded support on $x$: $\|x\| \leq R$ and 4) $\mu =\Omega(\frac{\gamma}{n^4})$. Let $w_t$ and $w_t'$ be the outputs of SGD on $S$ and $S'$ after $t$ steps, respectively. Let the divergence $\Delta_{t}:=w_t-w_t'$ and $\alpha\leq \frac{\mu
 }{2\beta^2 R^2}$ be the step size of SGD.
Then,
$$\mathbb{E}_{\mathbb{S}}\mathbb{E}_{\mathcal{A}}\|\Delta_T\| \leq  \frac{4LR}{\xi\gamma n}, \ \ \text{and} \ \  \varepsilon_{stab} (\mathcal{D}) \leq \frac{16L^2R^2}{\xi\gamma n}.$$
\end{thm}
\begin{proof}
 For simplicity we omit the dependence of $f$ on $y_j$ so that $f_{y_j} (w^\top x_j) = f(w,z_j)$.
Note that the gradient of the loss function is $\nabla f_{y_j}(w_t ^\top x_j) = f_{y_j}'(w_t ^\top x_j) x_j$ and the Hessian is $\nabla^2 f_{y_j}(w_t ^\top x_j) = f_{y_j}''(w_t ^\top x_j) x_jx_j^\top$. 
The stochastic gradient step of $f_{y_j}(w_t^\top x_j )$ is
$$w_{t+1} = w_t - \alpha_t f_{y_j}'(w_t^\top x_{j} ) x_{j}.$$
The dynamics of the divergence can be described as:
\begin{equation}\label{eq4}
\begin{aligned}
\mathbb{E}_{S,1:t+1} \|\Delta_{t+1}\| & =\mathbb{E}_{S}\mathbb{E}_{1:t} [ \frac{1}{n}\sum_{j \neq i} \|\Delta_t- \alpha_t [f'_{y_j}(w_t^\top x_j)-f'_{y_j}({w_t'}^\top x_j)] x_j \| \\
& +\frac{1}{n} \|\Delta_t-\alpha_t [f'_{y_i}(w_t^\top x_i)x_i-f'_{y_i'}(w_t'^\top x_i)x_i']\|  ] \\
\end{aligned}
\end{equation} 
Note that $[f'_{y_j}y_j(w_t^\top x_j)-f'_{y_j}(w_t'^\top x_j)]x_j$ can be rewritten as $f''_{y_j}({w_t^{\theta_j}} ^\top x_j)x_jx_j^\top\Delta_t    $ where $w_t^{\theta_j} = (1-\theta_j) w_t +\theta_j w_t', 0<\theta_j<1$.  
Similarly we can also rewrite 
$f'_{y_i}(w_t^\top x_i)x_i-f'_{y_i'}(w_t'^\top x_i)x_i'$ as
\begin{equation}\label{re_write}
\begin{aligned}
f'_{y_i}(w_t^\top x_i)x_i-f'_{y_i'}(w_t'^\top x'_i)x_i' & =\frac{1}{2}\{f'_{y_i}(w_t^\top x_i)x_i-f'_{y_i}(w_t'^\top x_i)x_i\} +\frac{1}{2}\{f'_{y_i'}(w_t^\top x_i')x_i'-f'_{y_i'}(w_t'^\top x_i')x_i'\}\\
 & \;\;\;\;+ \frac{1}{2}\{f'_{y_i} (w_t'^\top x_i)+f'_{y_i}(w_t^\top x_i)\}x_i-\frac{1}{2}\{f'_{y_i'}(w_t^\top x_i')+f_{y_i'}(w_t'^\top x_i')\}x_i' \\
& =\frac{1}{2}f''_{y_i}({w_t^{\theta_i}} ^\top x_i) x_i x_i^\top\Delta_t +\frac{1}{2}f''_{y'_i}({w_t^{{\theta'}_i}} ^\top x'_i) {x'_i x'_i} ^ \top \Delta_t \\
 & \;\;\;\;+ \frac{1}{2}\{f'_{y_i} (w_t'^\top x_i)+f'_{y_i}(w_t^\top x_i)\}x_i -\frac{1}{2}\{f'_{y_i'}(w_t^\top x_i')+f_{y_i'}(w_t'^\top x_i')\}x_i'\\
\end{aligned}
\end{equation} 
Let $\mathcal{H}_j = x_j x_j^\top $, $\mathcal{H}_i = \frac{1}{2}\{x_i x_i^\top + x'_i{x'_i}^\top\}$ and $\mathcal{H} = \frac{1}{n}\sum_{j} \mathcal{H}_j$.

Next we show the gradient of term $\frac{\mu}{2}w^\top w$ is bounded. This is because $w_{t+1} = (1-\alpha_t \mu)w_t -\alpha_t f'(w_t^\top) x_j$ which implies that $\|w_{t}\| \leq \frac{R L}{\mu}$ which implies that $\mu\|w\|\leq RL$.

By Equation~\ref{re_write}, Equation~\ref{eq4} can be written as 
\begin{equation*}
    \begin{aligned}
    &\mathbb{E}_{S}\mathbb{E}_{1:t} \left[ \frac{1}{n}\sum_{j \neq i} \|(1-\alpha_t \mu)\Delta_t- \alpha_t [f'_{y_j}(w_t^\top x_j)-f'_{y_j}({w_t'}^\top x_j)] x_j \| 
    +\frac{1}{n} \|(1-\alpha_t \mu)\Delta_t-\alpha_t [f'_{y_i}(w_t^\top x_i)x_i-f'_{y_i'}(w_t'^\top x_i)x_i']\|  \right] \\
    & \;\;\;\;\leq \mathbb{E}_{S}\mathbb{E}_{1:t} \big[ \frac{1}{n}\sum_{j \neq i} \|((1-\alpha_t \mu)I- \alpha_t f''_{y_j}({w_t^{\theta_j}} ^\top x_j) x_jx_j ^\top) \Delta_t \|\\
    & \;\;\;\; \;\;\;\;+ \frac{1}{n} \|((1-\alpha_t \mu)I-\frac{\alpha_t}{2} [f''_{y_i}({w_t^{\theta_i}} ^\top x_i) x_i x_i^\top +f''_{y'_i}({w_t^{{\theta'}_i}} ^\top x'_i) {x'_i x'_i} ^ \top ])\Delta_t\|   \big] +\frac{2 \alpha_t LR}{n}\\
    &\;\;\;\;\leq\mathbb{E}_{S}\mathbb{E}_{1:t}  \frac{1}{n}\sum_{j } \|((1-\alpha_t \mu)I- \alpha_t \gamma \mathcal{H}_j) \Delta_t \| \big] +\frac{2 \alpha_t LR}{n}\\
    \end{aligned}
\end{equation*}
\begin{equation*}
    \begin{aligned}
    &\leq\mathbb{E}_{S}\mathbb{E}_{1:t} \big[  \sqrt{\frac{1}{n}\sum_{j } \|((1-\alpha_t \mu)I- \alpha_t \gamma \mathcal{H}_j) \Delta_t \|^2 }\big] +\frac{2 \alpha_t LR}{n}\\
    & = \mathbb{E}_{S}\mathbb{E}_{1:t}  \sqrt{\frac{1}{n}\sum_{j } \left[  (1-\alpha_t \mu)^2\|\Delta_t\|^2 - 2(1-\alpha_t \mu)\alpha_t \gamma \Delta_t ^\top \mathcal{H}_j \Delta_t + \alpha_t^2\gamma^2\|\mathcal{H}_j \Delta_t\|^2 \right] } +\frac{2 \alpha_t LR}{n}\\
    &\leq \mathbb{E}_{S}\mathbb{E}_{1:t} \sqrt{  (1-\alpha_t \mu)^2\|\Delta_t\|^2 - 2\alpha_t(1-\alpha_t \mu)\gamma \Delta_t ^\top \mathcal{H} \Delta_t + \alpha_t^2\beta^2 R^2 \|\Delta_t\|^2 } +\frac{2 \alpha_t LR}{n}\\
    &\leq \mathbb{E}_{S}\mathbb{E}_{1:t} \sqrt{  (1-\alpha_t \mu)\|\Delta_t\|^2 - 2\alpha_t(1-\alpha_t \mu)\gamma \Delta_t ^\top \mathcal{H} \Delta_t -\alpha_t (\mu-\alpha_t\mu^2-\alpha_t\beta^2R^2) \|\Delta_t\|^2 } +\frac{2 \alpha_t LR}{n}\\
    &\leq \mathbb{E}_{S}\mathbb{E}_{1:t} \sqrt{  (1-\alpha_t \mu)\|\Delta_t\|^2 - 2\alpha_t(1-\alpha_t \mu)\gamma \Delta_t ^\top \mathcal{H} \Delta_t } +\frac{2 \alpha_t LR}{n}\\
    &\underset{*}{\leq} \mathbb{E}_{S}\mathbb{E}_{1:t}\sqrt{(1-\alpha_t \mu-2\alpha_t\gamma \xi_{S} ) \|\Delta_t\|^2} +\frac{2 \alpha_t LR}{n}\\
    &\leq \mathbb{E}_{S}\mathbb{E}_{1:t} \left[ (1-\frac{\alpha_t(\gamma\xi_S+\mu)}{2} )\|\Delta_t\| \right] +\frac{2 \alpha_t LR}{n}\\
    &\leq  \mathbb{E}_{S} \left[(1 - \frac{\alpha_t (\gamma \xi_S+\mu)  }{2}) \mathbb{E}_{1:t}\|\Delta_t\|+\frac{2\alpha_t LR}{n} \right]
    \end{aligned}
\end{equation*}

where in inequality $(*)$ we apply the fact that $\xi_S$ Rayleigh quotient bounded condition implied $\Delta_t \mathcal{H} \Delta_t \geq \xi_S \|\Delta_t \|^2$ since $\Delta_t \in Span\{x_1,..,x_i,x'_i,...,x_n\}$. Fix $\alpha_t = \alpha$ we have 
\begin{equation*}
    \mathbb{E}_{S} \left[\|\Delta_t\|\right] \leq  \mathbb{E}_{S} \left[\frac{4LR}{n(\gamma \xi_S+\mu)} \right] \leq \frac{4LR}{n(\gamma \xi+\mu)} \leq \frac{4LR}{n\gamma \xi}
\end{equation*}

and the theorem follows.
\end{proof}

\begin{claim}\label{old_lemma_3}
	Suppose $x_{t_0}=0$, $x_{t+1} = (1+\frac{a}{0.99t})x_t+ \frac{y}{t}$, we have $x_T \geq y (\frac{T}{t_0})^{a}$ if $a>0$ is a sufficiently small constant.
\end{claim}
\begin{proof}
In the proof we use following inequality:
$$e^{a} \leq 1 + \frac{a}{0.99} \leq e^{\frac{a}{0.99}}$$
where $a>0$ is a sufficiently small constant. 
\begin{equation}
\begin{aligned}
x_T&= \sum_{t=t_{0}+1}^{T} \frac{y}{t}\prod_{s=t+1}^{T} (1+\frac{a}{0.99s})\\
&\geq \sum_{t=t_{0}+1}^{T} \frac{y}{t} \exp \left({a\sum_{s=t+1}^{T} \frac{1}{s}}\right)\\
&\geq \sum_{t=t_{0}+1}^{T} \frac{y}{t} \exp \left( a\log(T/t) \right)\\
& \geq yT^{a}\sum_{t=t_{0}+1}^{T} \frac{1}{t^{1+a}}\\
& \geq y\left(\frac{T}{t_0}\right)^a 
\end{aligned}
\end{equation}
\end{proof}

\begin{lemma}[Divergence of non-convex loss function]
There exists a function $f$  which is non-convex and $\beta$-smooth,  twin datasets $S,S'$ and constant $a>0$ such that the following holds: if SGD is run using step size $\alpha_t = \frac{a}{0.99 \beta t}$ for $1 \leq t < T$, and $w_t, w_t'$ are the outputs of SGD on $S$ and $S'$, respectively, and $\Delta_{t}:=w_t-w_t'$, then:
\begin{equation} \forall 1 \leq t_{0} \leq T,\ \ \ \ 
       \mathbb{E}_{\mathcal{A}}\left[ \|\Delta_T\| | \Delta_{t_0} \neq 0  \right] \geq \frac{1}{2n} \left(\frac{T}{t_0}\right)^a 
\end{equation}
\end{lemma}

\begin{proof}
Consider the function $f(w,z)= \frac{1}{2} w^\top A w -yw^\top x$ , and choose $A$ to have positive and negative eigenvalues. We set the minimum eigenvalue of $A$ equal to $-\beta$ and all other eigenvalues with absolute value at most $\beta$.  We select twin datasets for such $A$ as follows.
We set all elements in $S \setminus \{x_i\} =S' \setminus\{x_i'\}$ to lie in the column space of $A$. Also, $\forall j \neq i$, choose $x_{j}$ such that $x_j^\top Ax_j >0$, and choose any $y_{j}$ equals 0.5.

Let $v$ be such that $v^\top Av= -\beta $ and $\|v\| =1$. Finally, let $x_i =v, y_i =0.5$, $x_i' = -v, y_i'=0.5$.

In this setting, one observes that the divergence $\Delta_{t}$ follows the following dynamic:


$$
\Delta_{t+1}= \left\{
        \begin{array}{ll}
          (I-\alpha_t A)\Delta_{t}   & \quad  \text{with prob. } 1-\frac{1}{n} \\
            (I-\alpha_tA)\Delta_{t} + \frac{\alpha_t}{2}[x_i-x_i'] & \quad \text{with prob } 1/n.
        \end{array}
    \right\}.
$$
We first observe that $\Delta_{t}:= w_t-w_t'$ is of the form  $v \theta_t$, where $\theta_{t}>0$.  This can be shown using induction. Let $\tau$ be the first time that $x_i,x_i'$ are picked, we have $\Delta_{\tau+1}= \frac{\alpha_{\tau}}{2}[x_i-x_i'] = v \alpha_{\tau}$. The iterative step of $\Delta_{t+1}$ and $\Delta_{t}$ implies that $\Delta_{t+1}= v \theta_{t+1}$ where $\theta_{t+1} = (1+\alpha_t \beta) \theta_t$ with probability $(1-\frac{1}{n})$ and $\theta_{t+1} = (1+\alpha_t \beta) \theta_t+ \alpha_t$ with probability $\frac{1}{n}$ .

The above construction then yields:
\begin{equation}
\begin{aligned}
    \mathbb{E}_{1:{t+1}}\left[ \|\Delta_{t+1}\| | \Delta_{t_0} \neq 0  \right] = &\mathbb{E}_{1:{t}} \left[\left( 1-\frac{1}{n}\right)\|(I-\alpha_t A)\Delta_{t}\| + \frac{1}{n} \| (I- \alpha_t A)\Delta_{t} + \alpha_t v\| \right]\\ 
    & =\|v\| \mathbb{E}_{1:{t}}   \left[ \left( 1-\frac{1}{n}\right)(1+\alpha_t \beta )\theta_t+ \frac{1}{n} ((1+\alpha_t \beta)\theta_t+ \alpha_t )  \right] \\
    & = \|v\| \mathbb{E}_{1:{t}} \left[\left[ (1+\alpha_t \beta )\theta_t \right] + \frac{ \alpha_t}{n}   \right]\\
    & = (1+\frac{a}{0.99 t})\mathbb{E}_{1:t}[\|\Delta_t\| | \Delta_{t_0} \neq 0]+\frac{\alpha_t}{n}\|v \|\\
\end{aligned}
\end{equation}
Now apply Claim \ref{old_lemma_3}, with $x_{t}= \mathbb{E} [\vert \vert \Delta_{t}\vert \vert | \Delta_{t_0} \neq 0 ]$ and $y= \frac{a \vert \vert v \vert \vert}{0.99 \beta n}$. This gives us that $x_{T} \geq \frac{a \vert \vert v \vert \vert}{0.99 \beta n} \left(T/t_{0} \right)^{a} = \frac{a }{0.99 \beta n} \left(T/t_{0} \right)^{a}$, since $\vert\vert v \vert \vert =1$.

Finally, the claimed bound follows by setting the minimum eigenvalue $\beta = \frac{a}{0.99}$.

\end{proof}

\begin{thm}[Lower bound for non-convex loss functions]
Let $w_t, w_t'$ be the outputs of SGD on twin datasets $S,S'$, and $\Delta_{t}:=w_t-w_t'$. There exists a function $f$ which is non-convex and $\beta$-smooth, twin datasets $S,S'$ and  constants $a<0.1$ such that the divergence of SGD after $T$ rounds ($n<T$) using constant step size $\alpha_t = \frac{a}{0.99\beta t}$ satisfies:
$$\varepsilon_{stab} >\frac{T^{a}}{6n^{1+a}}$$
\end{thm}

\begin{proof}
Follow the same construction of $f(w; z)$ and $S$, $S'$ in Lemma~\ref{thm2}.

We begin the  proof with Lemma~\ref{thm2} plus the idea of a ``burn-in" period. We have:\\
\begin{equation}
    \begin{aligned}
    \mathbb{E}_{\mathcal{A}}\|\Delta_T\| & = \mathbb{E}_{\mathcal{A}}[\|w_t-w_t'\|| \Delta_n = 0]\mathbb{P}[\Delta_n =0]+ \mathbb{E}_{\mathcal{A}}[\|w_t-w_t'\|| \Delta_n \neq 0]\mathbb{P}[\Delta_n \neq 0]\\
     & \geq  \mathbb{E}_{\mathcal{A}}[\|w_t-w_t'\|| \Delta_n \neq  0]\mathbb{P}[\Delta_n \neq  0]\\
  & = \left(1-\big(1-\frac{1}{n}\big)^n\right)\frac{T^a}{2n^{1+a}}\|x_i-x_i'\|\\
  & > \frac{T^a}{6n^{1+a}}\|x_i-x_i'\|
    \end{aligned}
\end{equation}
By a similar proof as in Theorem~\ref{strongly_cvx_lowerbound} we can show  $w_t^\top v = -{w'_t}^\top v$ thus  $f(w_T;z)-f(w_T';z) = z^\top [w_T-w_T']$ and by restricting $z\sim \mathbb{Z}$ where $ \mathbb{Z}$ is the linear span of eigenvectors of $A$, we have $$\sup\limits_{z} \mathbb{E}_{\mathcal{A}} [f(w_T;z)-f(w'_T;z)] = \mathbb{E}_{\mathcal{A}}  [w_T-w'_T]^\top v = \theta_t> \frac{T^a}{6n^{1+a}}$$
\end{proof}

\begin{lemma} \cite{hardt2016train}
	Assume $f$ is $\beta$-smooth and $L$-lipschitz. Let $w_t, w_t'$ be outputs of $SGD$ on twin datasets $S,S'$ respectively after $t$ iterations and let $\Delta_{t}:=[w_t-w_t']$ and $\delta_t = \mathbb{E}_{\mathcal{A}}\|\Delta_t\|$. Running SGD on $f(w;S)$ with step size $\alpha_t = \frac{a}{\beta t }$ satisfies the following conditions:
	\begin{itemize}
		\item 
		The SGD update rule is a $(1+\alpha_t \beta)$-expander and $2\alpha_t L$-bounded. 
		\item 
		$\mathbb{E}_{\mathcal{A}}[ \|\Delta_{t}\| | \Delta_{t-1} ]  \leq \left(1+\alpha_t\delta\right)\|\Delta_{t-1}\| +\frac{2\alpha_tL}{n}$.
		\item 
		$\mathbb{E}_{\mathcal{A}}[ \|\Delta_{T}\| | \Delta_{t_{k}}=0 ] \leq \big(\frac{T}{t_{k}}\big)^{ a}\frac{2L}{ n}$.
	\end{itemize}
\end{lemma}

\begin{thm}[Permutation]
	Assume $f$ is $\beta$-smooth and $L$-lipschitz. Running $T$ ($T>n$) iterations of  SGD  on $f(w;S)$ with step size $\alpha_t = \frac{a}{\beta t }$, the stability of SGD satisfies:\\ 
	\begin{equation}
	\mathbb{E}_{\mathcal{A}}\|\Delta_T\|   \leq\frac{ 2L T^a}{n^{1+a}}, \varepsilon_{stab} \leq  \frac{2L^2T^a}{n^{1+a}}
	\end{equation}
\end{thm}

\begin{proof}

Let $H=t$ represents the event that the first time the SGD pick the different entry is at time $t$:\\
\begin{equation}
    \begin{aligned}
        \mathbb{E}_{\mathcal{A}}\|\Delta_T\|&  =  \mathbb{E}_{\mathcal{A}}[\|\Delta_T\|| H \leq n] \mathbb{P} [H \leq n] + \underbrace{\mathbb{E}_{\mathcal{A}}[\|\Delta_T\|| H > n] \mathbb{P} [H > n]}_{0 \text{(permutation)}}\\
        &\leq \frac{1}{n}\sum_{t=1}^{n} \mathbb{E}_{\mathcal{A}}[\|\Delta_T\| | H = t]\\
        &\underset{*}{\leq} \frac{1}{n}\sum_{t=1}^{n} \left(\frac{T}{t}\right)^a\frac{2L}{n} \\
        &\leq \frac{2LT^a}{n^2}\int_{t=1}^{n}\frac{1}{t^a} dt\\
        &\leq \frac{2L T^a}{n^{1+a}}
    \end{aligned}
\end{equation}
The inequality $(*)$ derived by applying Lemma \ref{lem_SGD}.
\end{proof}

\begin{lemma} 
	Let $w_t, w_t'$ be outputs of $SGD$ on twin datasets $S,S'$ respectively after $t$ iterations and let $\Delta_{t}:=w_t-w_t'$. Suppose that $t_k = ct_{k-1}$. Then the following conditions hold:
	
	\begin{itemize}
		\item 
		$\mathbb{P}[ \Delta_{t_k-1}=0| \Delta_{t_{k}}\neq 0 ]\leq \frac{n}{n+t_{k-1}}$.
		\item 
		$\mathbb{P}[ \Delta_{t_k-1} \neq0| \Delta_{t_{k}}\neq 0 ]\leq \frac{1}{c}\left(1+\frac{t_k}{n}\right)$.
		\item 
		$	\mathbb{E}_{\mathcal{A}}[ \|\Delta_{T}\|| \Delta_{t_{k}}\neq 0 ]
		\leq \mathbb{E}_{\mathcal{A}}[ \|\Delta_{T}\| | \Delta_{t_{k-1}}\neq 0 ]  \frac{1}{c}\left(1+\frac{t_k}{n}\right)+\big(\frac{T}{t_{k-1}}\big)^{ a}\frac{2L}{ n}$.
	\end{itemize}
\end{lemma}

\begin{proof} 
In the proof we will use the following inequality with $r\geq1$:
$$\frac{n-r}{n} \leq (1-\frac{1}{n})^{r }\leq \frac{n}{n+r}$$

\noindent i): 

\begin{equation}
\begin{aligned}
& \mathbb{P}[ \Delta_{t_{k-1}}=0| \Delta_{t_{k}}\neq 0 ]=\frac{\mathbb{P}[ \Delta_{t_k-1}=0, \Delta_{t_{k}}\neq 0 ]}{\mathbb{P}[\Delta_{t_{k}}\neq 0]}\\
&=(1-1/n)^{t_{k-1}}\frac{1-(1-1/n)^{t_k-t_{k-1}}}{1-(1-1/n)^{t_k}}\leq (1-1/n)^{t_{k-1}} \leq \frac{n}{n+t_{k-1}}
\end{aligned}
\end{equation}  

\noindent ii):	
\begin{equation}
\begin{aligned}
&\mathbb{P}[ \Delta_{t_k-1} \neq0| \Delta_{t_{k}}\neq 0 ]=\frac{\mathbb{P}[ \Delta_{t_k} \neq0, \Delta_{t_{k-1}}\neq 0 ]}{\mathbb{P}[ \Delta_{t_k} \neq0]}\\
&=\frac{\mathbb{P}[\Delta_{t_{k-1}}\neq 0 ]}{\mathbb{P}[ \Delta_{t_k} \neq0]}=\frac{1-(1-1/n)^{t_{k-1}}}{1-(1-1/n)^{t_k}}\\
&\leq \frac{1-\frac{n}{n+t_{k-1}}}{1-\frac{n-t_k}{n}}\leq \frac{t_{k-1}}{t_k}(1+\frac{t_k}{n})\\
&=\frac{1}{c}(1+\frac{t_k}{n})
\end{aligned}
\end{equation}  

\noindent iii): By applying i) and ii) in the decomposition of $\mathbb{E}[ \Delta_{T} | \Delta_{t_{k}}\neq 0 ]$ we have
\begin{equation}
\begin{aligned}
\mathbb{E}_{\mathcal{A}}[ \|\Delta_{T}\| | \Delta_{t_{k}}\neq 0 ]&\leq \mathbb{E}_{\mathcal{A}}[ \|\Delta_{T}\| | \Delta_{t_{k-1}}\neq 0 ]  \mathbb{P}[ \Delta_{t_k-1} \neq0| \Delta_{t_{k}}\neq 0 ]+ \mathbb{E}_{\mathcal{A}}[ \|\Delta_{T}\| | \Delta_{t_{k-1}}= 0 ] \mathbb{P}[ \Delta_{t_k-1}=0| \Delta_{t_{k}}\neq 0 ]\\
&\leq \mathbb{E}_{\mathcal{A}}[ \|\Delta_{T}\| | \Delta_{t_{k-1}}\neq 0 ]  \frac{t_{k-1}}{t_k}(1+\frac{t_k}{n})+(\frac{T}{t_{k-1}})^{a}\frac{2L}{n+t_{k-1}}\\
&=\frac{1}{c}(1+\frac{t_k}{n})\mathbb{E}_{\mathcal{A}}[ \|\Delta_{T}\| | \Delta_{t_{k-1}}\neq 0 ]  +(\frac{T}{t_{k-1}})^{a}\frac{2L}{n+t_{k-1}}
\end{aligned}
\end{equation}  
where the last inequality uses the fact that $\mathbb{E}_{\mathcal{A}}[ \|\Delta_{T}\| | \Delta_{t_{k}}=0 ] \leq\big(\frac{T}{t_{k-1}}\big)^{ a}\frac{2L}{ n}$.
\end{proof}
\begin{thm}[Uniformly Sampling SGD]
	Assume $f$ is $\beta$-smooth and $L$-lipschitz. Running $T$ ($T>n$) iterations of SGD  on $f(w;S)$ with step size $\alpha_t = \frac{a}{\beta t }$, the stability of SGD satisfies:\\ 
	\begin{equation}
	\mathbb{E}_{\mathcal{A}} \|\Delta_T\|   \leq 16 \log(n)L\frac{T^a}{n^{1+a}};\;\;\;\varepsilon_{stab} \leq  16\log(n)L^2\frac{T^a}{n^{1+a}}
	\end{equation}
\end{thm}

\begin{proof}

We first decompose $\Delta_T$ as follows by selecting $t_k=n$:\\
\begin{equation}
\begin{aligned}
\mathbb{E}_{\mathcal{A}}\|\Delta_T\|=&\underbrace{\mathbb{E}_{\mathcal{A}}[\|\Delta_T\| | \Delta_{t_k}=0] \mathbb{P}[\Delta_{t_k}=0]}_\text{Term 1 $ \leq \frac{2LT^{a  }}{n^{1+a  }} $ (Lemma~\ref{lem_SGD})}+\underbrace{\mathbb{E}_{\mathcal{A}}[\|\Delta_T\| | \Delta_{t_k}\neq 0] \mathbb{P}[\Delta_{t_k}\neq 0]}_\text{Term 2 $\leq \frac{11L \log(n)T^ {a  } }{n^{1+a  }}$} 
\end{aligned}
\end{equation}
Term 1 is easily bounded by applying Lemma~\ref{lem_SGD} with $\alpha_t=\frac{a}{t\beta}$. To bound Term 2, plug in $\mathbb{P}[\Delta_{t_k} \neq 0]= 1-(1-1/n)^{t_k} \leq \frac{t_k}{n}$ and recursively apply point (iii) from Lemma~\ref{lem_prob_rule} by setting $t_{i+1}=c t_i$. We get:\\
\begin{equation}
\begin{aligned}
&\mathbb{E}_{\mathcal{A}}[\|\Delta_T\|| \Delta_{t_k} \neq 0]\mathbb{P}[\Delta_{t_k} \neq 0]\\
&\leq\frac{2L}{n} \frac{t_k}{n}\sum_{i=1}^{k-1} (\frac{T}{t_i})^{  a}\frac{n}{n+t_i} \prod_{\tau =i+1}^{k-1} (1+\frac{t_{\tau+1}}{n})\frac{t_\tau}{t_{\tau+1}}\\
&\leq \frac{2L}{n}\sum_{i=1}^{k-1} (\frac{T}{t_i})^{  a}\frac{t_{i+1}}{n+t_i}
exp(\sum_{\tau =i+1}^{k-1}\frac{t_{\tau+1}}{n})	\\
& \leq\frac{2cL}{n}exp\left( \frac{c}{c-1}\right)\sum_{i=1}^{k-1} (\frac{T}{t_i})^{  a}\frac{t_{i}}{n+t_i}	\\
& \leq\frac{2cL T^a }{n}exp\left( \frac{c}{c-1}\right)\sum_{i=1}^{k-1}\frac{t_{i}^{1-a}}{n}\\
& \leq\frac{2L \log(n)T^a}{n^{1+a}} \frac{c^a}{\log c}exp\left( \frac{c}{c-1}\right)\\
&\leq \frac{11\log (n)L T^{  a}}{n^{1+  a}}
\end{aligned}
\end{equation}
In the second and third inequality we use the fact that $1+x \leq \exp(x)$ and $t_{i+1}=c t_i$ to get $\prod_{\tau =i+1}^{k-1} (1+\frac{t_{\tau+1}}{n}) \leq exp(\sum_{\tau =i+1}^{k-1}\frac{t_{\tau+1}}{n}) \leq exp\left( \frac{c}{c-1}\right) $ . The last inequality is derived by  picking $c = 4$ .

\end{proof}

\begin{lemma}
    Assume $f$ is $\beta$-smooth $L$-Lipschitz and $\rho$-Lipschitz Hessian. Let $w_0$ be the initialization weight and $w_t$, $w_{t'}$ be the outputs of SGD on twin datasets $S$ and $S'$ respectively after $t$ iterations. Let $\Delta_t :=[w_t - w_{t'}]$. Running SGD on $f(w;S)$ with step size $\alpha_t = \frac{b}{t}$ satisfies $b \leq min\{\frac{2}{\beta}, \frac{1}{8\beta^2\ln T^2}\}$ has the following properties:
    \begin{enumerate}
        \item The SGD update rule is a $(1+\alpha_t\psi_t)$-expander and a $\alpha_t L$-bounded. Here $\psi_t = min\{\beta, \kappa_t\}$ where 
        \[
        \kappa_t = \|\nabla^2 f(w_0, z_t)\|_2 + \frac{\rho}{2}\|\sum_{k=1}^{t-1}\alpha_k\nabla f(w_{S, k}, z_k)\| + \frac{\rho}{2}\|\sum_{k=1}^{t-1}\alpha_k\nabla f(w_{S^{'}, k}, z_k)\|
        \]
        \item $E_{\mathcal{A}}[\|\Delta_{t+1}\||\Delta_{t_0}=0]\leq [1+(1-1/n)\alpha_t\psi_t]E_{\mathcal{A}}[\|\Delta_{t}\||\Delta_{t_0}=0] + \frac{2\alpha_tL}{n}$.
        \item $E_{S}\{E_{\mathcal{A}}[\|\Delta_{T}\||\Delta_{t_0}=0]\} \leq \frac{L}{n}\left(\frac{T}{t_0}\right)^{\zeta b}$, where
        \[
        \zeta := \tilde{O}(\min\{\beta, E_z[\|\nabla^2f(w_0, z)\|_2] + \Delta^*_{1, \sigma^2}\})
        \]
        \[
        \Delta^*_{1, \sigma^2} = \rho(b\sigma + \sqrt{bE_z[f(w_0, z)] - \inf_{w}E_z[f(w, z)]})
        \]
    \end{enumerate}
\end{lemma}
\begin{proof}
    \begin{enumerate}
        \item We could find this results in \cite{kuzborskij2018data} equation (16).
        \item According to equation (19) in \cite{kuzborskij2018data} we could have this conclusion.
        \item In \cite{kuzborskij2018data}'s proof of theorem part 3, we could obtain this inequality.
    \end{enumerate}
\end{proof}

\begin{thm}[Data-dependent version of Theorem~\ref{permutation}]
    Assume $f$ is $\beta$-smooth $L$-Lipschitz and $\rho$-Lipschitz Hessian. Let $w_0$ be the initialization weight and $w_t$, $w_{t'}$ be the outputs of SGD on twin datasets $S$ and $S'$ respectively after $t$ iterations. Let $\Delta_t :=[w_t - w_{t'}]$ and $\delta_t = E_{\mathcal{A}}\|\Delta_t\|$. Running SGD on $f(w;S)$ with step size $\alpha_t = \frac{b}{t}$ satisfies $b \leq min\{\frac{2}{\beta}, \frac{1}{8\beta^2\ln T^2}\}$ has the following properties:\\ 
	\begin{equation}
	\mathbb{E}_S[\delta_T]   \leq\frac{ L T^{\zeta b}}{\zeta n^{1+\zeta b}}, \widehat{\varepsilon_{stab}}(\mathcal{D}, w_0) \leq  \frac{L^2T^{\zeta b}}{\zeta n^{1+\zeta b}}
	\end{equation}
\end{thm}
\begin{proof}
\begin{equation}
\begin{aligned}
        \mathbb{E}_S[\delta_T]&  =  \mathbb{E}_S[\delta_T| H \leq n] \mathbb{P} [H \leq n] + \underbrace{\mathbb{E}_S[\delta_T| H > n] \mathbb{P} [H > n]}_{0 \text{(permutation)}} \\
        &\leq \frac{1}{n}\sum_{t=1}^{n} \mathbb{E}_S[\delta_T | H = t]\\
        &{\leq} \frac{1}{n}\sum_{t=1}^{n} \left(\frac{T}{t}\right)^{\zeta b}\frac{L}{\zeta n} \\
        &\leq \frac{LT^{\zeta b}}{n^2}\int_{t=1}^{n}\frac{1}{t^{\zeta b}} dt\\
        &\leq \frac{L T^{\zeta b}}{\zeta n^{1+\zeta b}}
    \end{aligned}
\end{equation}
\end{proof}

\begin{lemma}[Data-dependent version of Lemma~\ref{lem_prob_rule}]
	Let $w_t, w_t'$ be outputs of $SGD$ on twin datasets $S,S'$ respectively after $t$ iterations and let $\Delta_{t}:=w_t-w_t'$. And $b$, $\zeta$ follow the same definition in Lemma \ref{lem_SGD_data_dep}. Suppose that $t_k = ct_{k-1}$. Then the following condition holds:
		$$\mathbb{E}_S\mathbb{E}_{\mathcal{A}}[ \|\Delta_{T}\|| \Delta_{t_{k}}\neq 0 ]
		\leq \mathbb{E}_S\mathbb{E}_{\mathcal{A}}[ \|\Delta_{T}\| | \Delta_{t_{k-1}}\neq 0 ]  \frac{1}{c}\left(1+\frac{t_k}{n}\right)+\big(\frac{T}{t_{k-1}}\big)^{\zeta b}\frac{L}{\zeta n}$$
\end{lemma}
\begin{proof}
\begin{equation}
    \begin{aligned}
        &\mathbb{E}_S\mathbb{E}_{\mathcal{A}}[ \|\Delta_{T}\|| \Delta_{t_{k}}\neq 0 ]  \leq \mathbb{E}_S\mathbb{E}_{\mathcal{A}}[ \|\Delta_{T}\|| \Delta_{t_{k-1}}\neq 0 ]\mathbb{P}[ \Delta_{t_k-1} \neq0| \Delta_{t_{k}}\neq 0 ] \\
        &+ \mathbb{E}_S\mathbb{E}_{\mathcal{A}}[ \|\Delta_{T}\| | \Delta_{t_{k-1}}= 0 ] \mathbb{P}[ \Delta_{t_k-1}=0| \Delta_{t_{k}}\neq 0 ] \\
        &\leq \mathbb{E}_S\mathbb{E}_{\mathcal{A}}[ \|\Delta_{T}\| | \Delta_{t_{k-1}}\neq 0 ]  \frac{t_{k-1}}{t_k}(1+\frac{t_k}{n})+(\frac{T}{t_{k-1}})^{\zeta b}\frac{L}{\zeta(n+t_{k-1})}\\
        &=\frac{1}{c}(1+\frac{t_k}{n})\mathbb{E}_S\mathbb{E}_{\mathcal{A}}[ \|\Delta_{T}\| | \Delta_{t_{k-1}}\neq 0 ] +(\frac{T}{t_{k-1}})^{\zeta b}\frac{L}{\zeta(n+t_{k-1})}
    \end{aligned}
\end{equation}
The second inequality follows Lemma~\ref{lem_SGD_data_dep}.
\end{proof}

\begin{thm}[Data-dependent version of Theorem~\ref{unif_indep}]
	Assume $f$ is $\beta$-smooth $L$-Lipschitz and $\rho$-Lipschitz Hessian. Let $w_t$ and $w_{t'}$ be the outputs of SGD on twin datasets $S$ and $S'$ respectively after $t$ iterations and let $\Delta_t :=[w_t - w_{t'}]$ and $\delta_t = E_A\|\Delta_t\|$. And $\zeta$ follows the same definition in Theorem \ref{permutation_new}. Running SGD on $f(w;S)$ with step size $\alpha_t = \frac{b}{t}$ satisfies $b < 1$ has the following properties:\\ 
    \begin{equation}
	\mathbb{E}_{S}\mathbb{E}_{\mathcal{A}} \|\Delta_T\|   \leq 16 \log(n)L\frac{T^{\zeta b}}{\zeta n^{1+\zeta b}};\;\;\;\widehat{\varepsilon_{stab}}(\mathcal{D}, w_1) \leq  16\log(n)L^2\frac{T^{\zeta b}}{\zeta n^{1+\zeta b}}
	\end{equation}
\end{thm}
	
\begin{proof}
We follow the assumption and proof  in Theorem~\ref{unif_indep}. To bound the Term 1 in Theorem~\ref{unif_indep}, we directly apply Lemma \ref{lem_SGD_data_dep}. To bound Term 2, we recursively apply Lemma \ref{lem_prob_rule_new} and set $t_{i+1} = ct_i$. We have
\begin{equation}
\begin{aligned}
&\mathbb{E}_{S}\mathbb{E}_\mathcal{A}[\|\Delta_T\|| \Delta_{t_k} \neq 0]\mathbb{P}[\Delta_{t_k} \neq 0]\\
&\leq\frac{L}{\zeta n} \frac{t_k}{n}\sum_{i=1}^{k-1} (\frac{T}{t_i})^{\zeta b}\frac{n}{n+t_i} \prod_{\tau =i+1}^{k-1} (1+\frac{t_{\tau+1}}{n})\frac{t_\tau}{t_{\tau+1}}\\
& \leq\frac{L \log(n)T^{\zeta b}}{\zeta n^{1+\zeta b}} \frac{c^{\zeta b}}{\log c}exp\left( \frac{c}{c-1}\right)\\
&\leq \frac{11\log (n)L T^{\zeta b}}{\zeta n^{1+b}}
\end{aligned}
\end{equation}
By applying $c = 4$ and following same procedure in proving Theorem~\ref{unif_indep} we obtain the last inequality.
Therefore, we could bound $\mathbb{E}_{S}\mathbb{E}_\mathcal{A}[\|\Delta_T\|]$ by adding two terms together and get
\begin{equation}
    \mathbb{E}_{S}\mathbb{E}_{\mathcal{A}} \|\Delta_T\|   \leq 16 \log(n)L\frac{T^{\zeta b}}{\zeta n^{1+\zeta b}}
\end{equation}
\end{proof}

\begin{thm} 
Let $w_t, w_t'$ be the outputs of SGD on twin datasets $S,S'$, and let $\Delta_{t}:=w_t-w_t'$. There exists a function $f$ which is non-convex and $\beta$-smooth, twin sets $S,S'$ and constants $a,\zeta$ such that the divergence of SGD after $T$ rounds ($T>n$) using constant step size $\alpha = \frac{a}{0.99 \zeta }$ satisfies:
\begin{equation}
    \varepsilon_{stab} \geq \frac{1}{n^2}e^{aT/2}
\end{equation}
\end{thm}
\begin{proof}
The proof is similar to Theorem \ref{thm3}. Since $\Delta_{t} \in Span\{x_i-x_i'\}$, we have:
$$\mathbb{E}_{\mathcal{A}}\|\Delta_{t+1}\| \geq (1-\frac{1}{n})(1+\alpha_t\beta)\mathbb{E}\|\Delta_t\| +\frac{\alpha_t}{n} \|x_i-x_i\|   $$
Suppose $t_0$ is the hitting time when $\|\Delta_{t_0}\| >0$ and $\|\Delta_{t_0-1}\| =0$ ,$\|\Delta_{T}\| \geq \frac{\|x_i-x_i'\|}{3n}e^{a(T-t_0)/2} $. 
\begin{equation}
    \begin{aligned}
    \mathbb{E}_{\mathcal{A}}\|\Delta_T\| &= \mathbb{E}_{\mathcal{A}}[\|w_t-w_t'\|| \Delta_1 = 0]\mathbb{P}[\Delta_1 =0]+ \mathbb{E}_{\mathcal{A}}[\|w_t-w_t'\|| \Delta_1 \neq 0]\mathbb{P}[\Delta_1 \neq 0]\\
    &\geq \mathbb{E}_{\mathcal{A}}[\|w_t-w_t'\|| \Delta_1 \neq  0]\mathbb{P}[\Delta_1 \neq  0]\\
    &= \frac{1}{n}( \frac{\|x_i-x_i'\|}{n}e^{aT/2})\\
    &= \frac{\|x_i-x_i'\|}{n^2}e^{aT/2}.
    \end{aligned}
\end{equation}
By a similar proof as Theorem ~\ref{thm3} one can obtain the stability lower bound.
\end{proof}

\end{appendices}

\end{document}